\documentclass{article}

\usepackage{times}

\title{Adam or Gauss-Newton? A Comparative Study In Terms of Basis Alignment and SGD Noise}

\usepackage[hidelinks]{hyperref}
\hypersetup{hidelinks}

\makeatletter
\newcommand\blfootnote[1]{%
  \begingroup
  \renewcommand\thefootnote{}%
  \footnotetext{#1}%
  \endgroup
}
\makeatother

\author{%
  Bingbin Liu$^{*,\dagger}$ \qquad
  Rachit Bansal$^{*,\dagger}$ \qquad
  Depen Morwani$^{*,\dagger}$ \\
  Nikhil Vyas$^{\alpha, \beta}$ \qquad
  David Alvarez-Melis$^{\dagger}$ \qquad
  Sham M. Kakade$^{\dagger}$ \\
  $^{\dagger}$ Kempner Institute at Harvard University \qquad $^{\alpha}$ OpenAI
}
\date{}

\usepackage[utf8]{inputenc} %
\usepackage{hyperref}       %
\usepackage{url}            %
\usepackage{booktabs}       %
\usepackage{nicefrac}       %
\usepackage{microtype}      %
\usepackage{array}
\usepackage{subcaption}   %
\usepackage{graphicx}     %

\usepackage{wrapfig}

\usepackage{amsmath}
\usepackage{amssymb}
\usepackage{amsthm}
\usepackage{bm,bbm}

\usepackage{listings}
\usepackage[dvipsnames]{xcolor}
\usepackage{cleveref}
\usepackage{mleftright}
\usepackage{enumitem}
\setlist[itemize,1]{leftmargin=1em,itemsep=0.1em,topsep=0.3em}

\usepackage{url}
\usepackage{xparse}
\usepackage{algorithm}
\usepackage{algorithmic}
\usepackage[left=1in,bottom=1in,top=1in,right=1in]{geometry}
\usepackage{titlesec}
\titlespacing*{\paragraph}{0pt}{1ex plus .5ex}{0.5ex}
\titleformat{\paragraph}[runin]
  {\normalfont\normalsize\bfseries}
  {\theparagraph}{1em}{}[~~]

\usepackage[bottom,flushmargin]{footmisc}

\usepackage{caption}
\usepackage{subcaption}

\usepackage{natbib}

\usepackage{comment}

\usepackage{ifxetex}
\ifxetex
\usepackage[bigdelims,vvarbb]{newpxmath}
\usepackage{mathspec}
\setallmainfonts(Digits,Latin){Noto Serif}
\else
\usepackage{mathpazo}
\fi

\def \dim {d}
\def \dimw {m}
\def \dimh {n}
\def \hiddendim {m}
\def \bt {b_G} %
\def \btH {b_H} %
\def \teacher {\gT}
\def \student {\gS}
\def \distr {\gD} %

\NewDocumentCommand{\loss}{o}{%
  \ell %
  \IfValueT{#1}{^{(#1)}}%
}
\NewDocumentCommand{\Loss}{o}{%
  \gL %
  \IfValueT{#1}{^{(#1)}}%
}
\NewDocumentCommand{\lr}{o}{%
  \eta %
  \IfValueT{#1}{^{(#1)}}%
}
\def \reg {\alpha} %
\def \interp {\gamma} %

\NewDocumentCommand{\param}{o}{%
  \theta %
  \IfValueT{#1}{^{(#1)}}%
}
\NewDocumentCommand{\Param}{o}{%
  \theta %
  \IfValueT{#1}{^{(#1)}}%
}
\NewDocumentCommand{\tparam}{o}{%
  \tilde{\theta} %
  \IfValueT{#1}{^{(#1)}}%
}
\NewDocumentCommand{\grad}{o}{%
  g
  \IfValueT{#1}{^{(#1)}}%
}
\NewDocumentCommand{\gradrot}{o}{%
  \tilde{g}%
  \IfValueT{#1}{^{(#1)}}%
}
\NewDocumentCommand{\Grad}{o}{%
  G
  \IfValueT{#1}{^{(#1)}}%
}

\NewDocumentCommand{\DeltaParam}{o}{%
  \Delta
  \IfValueT{#1}{^{(#1)}}%
}
\NewDocumentCommand{\update}{o}{%
  \delta
  \IfValueT{#1}{^{(#1)}}%
}
\NewDocumentCommand{\paramCov}{o}{%
  M %
  \IfValueT{#1}{^{(#1)}}%
}
\NewDocumentCommand{\tparamCov}{o}{%
  \widetilde{M} %
  \IfValueT{#1}{^{(#1)}}%
}
\def \Id {\mathbf{I}}
\def \Cov {\Sigma}
\def \GN {\text{GN}}
\def \hessian {H}
\def \hGN {H^{(\GN)}} %
\def \power {p} %
\NewDocumentCommand{\precond}{o}{%
  P%
  \IfValueT{#1}{^{(#1)}}%
}
\NewDocumentCommand{\precondD}{o}{%
  D%
  \IfValueT{#1}{^{(#1)}}%
}
\def \diag {\text{diag}}
\def \cond {\kappa} %
\def \eval {\lambda} %
\def \Eval {\Lambda} %
\def \ebasis {U} %
\def \evec {u} %

\def \spectralRad {\gamma}
\def \py {P}
\def \pf {q}
\def \seqlen {T}

\newcounter{algline} %
\newcommand{\alglinenumber}{%
  \refstepcounter{algline} %
  \arabic{algline}:        %
}

\definecolor{RoyalBlue}{RGB}{0,100,170}
\definecolor{peach}{rgb}{1, 0.56, 0.56}
\definecolor{midgray}{RGB}{150,150,150}
\definecolor{EasternBlue}{RGB}{37,150,190}
\definecolor{sand}{RGB}{250,150,120}
\definecolor{grass}{RGB}{120, 190, 50}
\definecolor{sky}{RGB}{50,150,250}
\definecolor{Orange}{RGB}{250,150,50}
\definecolor{Cerulean}{RGB}{80,150,220}
\definecolor{Emerald}{RGB}{62,156,94}
\definecolor{Rouge}{RGB}{250,95,95}
\definecolor{coral}{RGB}{240,128,128}

\definecolor{ColorDef}{RGB}{80, 180, 150}
\definecolor{RevisionRed}{RGB}{240,35,35}
\definecolor{RevisionBlue}{RGB}{80,180,250}
\definecolor{TODO}{RGB}{255,150,50}

\definecolor{shamcolor}{HTML}{1f77b4}   %
\definecolor{davidcolor}{HTML}{ffde21}  %
\definecolor{depencolor}{HTML}{2ca02c}     %
\definecolor{bingbincolor}{HTML}{f26b83}   %
\definecolor{rachitcolor}{HTML}{b027d6}  %

\renewcommand{\Pr}{\mathop{\bf Pr\/}}

\newtheorem{theorem}{Theorem}
\newtheorem*{namedtheorem}{\theoremname}
\newcommand{\theoremname}{testing}

\newtheorem*{theorem*}{Theorem}
\newtheorem{lemma}{Lemma}
\newtheorem*{lemma*}{Lemma}

\newtheorem{claim}{Claim}

\newtheorem{corollary}[theorem]{Corollary}
\newtheorem*{corollary*}{Corollary}
\newtheorem{cor}[theorem]{Corollary}

\newtheorem*{question*}{Question}

\theoremstyle{definition}

\newtheorem*{definition*}{Definition}

\newtheorem*{remark*}{Remark}

\theoremstyle{plain}

\def\eqref#1{equation~\ref{#1}}

\def\1{\bm{1}}

\def\ve{{\bm{e}}}

\def\vx{{\bm{x}}}

\def\mA{{\bm{A}}}

\DeclareMathAlphabet{\mathsfit}{\encodingdefault}{\sfdefault}{m}{sl}
\SetMathAlphabet{\mathsfit}{bold}{\encodingdefault}{\sfdefault}{bx}{n}

\def\gD{{\mathcal{D}}}

\def\gL{{\mathcal{L}}}

\def\gP{{\mathcal{P}}}

\def\gS{{\mathcal{S}}}
\def\gT{{\mathcal{T}}}

\newcommand{\E}{\mathbb{E}}

\newcommand{\R}{\mathbb{R}}

\DeclareMathOperator{\Tr}{Tr}

\hypersetup{colorlinks=true,linkcolor=OliveGreen,citecolor=OliveGreen,urlcolor=OliveGreen}
\begin{document}

\maketitle
\blfootnote{$^*$ Equal contribution. Correspondence to \texttt{\{bliu, rachitbansal, dmorwani\}@g.harvard.edu}.}
\blfootnote{$^{\beta}$ Work done as a PostDoc at Harvard.}

\begin{abstract}
  Preconditioned methods are central to deep learning optimization.
  Predominant approaches include computationally light diagonal preconditioners such as Adam, which rely on gradient statistics,
  and second-order methods such as Gauss-Newton (GN), which capture richer curvature information.
  Seeking the best of both worlds,
  we disentangle the preconditioner design space into several factors, 
  separating the choice of diagonal scaling (Adam-style versus GN-style) from 1) the \textit{basis choices} under which the diagonal scaling operates, and 2) the \textit{gradient noises} from mini-batching.
  Our theoretical results show that GN's optimality for linear regression no longer holds under a poorly chosen basis under both population and stochastic updates, or a move from linear regression to a non-convex variant of logistic regression.
  In these settings, Adam-style methods offer genuine advantages over curvature-inspired preconditioning, rather than serving merely as a tractable proxy.
  Empirical results on synthetic problems and CIFAR-10 support these findings.
\end{abstract}

\section{Introduction}
\label{sec:intro}

The heterogeneous optimization landscape in deep learning motivates the use of preconditioned methods~\citep{sagun2016eigen,ghorbani19eigen,yao2019pyhessian0,zhang2020adaptive,liu2023sophia},
with updates of the form
$
\param_{t+1} = \param_t - \lr \precond \grad_t
$.
On one hand, diagonal preconditioners, such as Adam~\citep{kingma2014adam} and related methods~\citep{tieleman2012rmsprop,shazeer2018adafactor,bernstein2018signsgd,chen2023lion} are favored for their computational efficiency and have become the de facto standard.
On the other hand, second-order preconditioners that capture curvature information,
most notably the Gauss-Newton (GN) method, are more expensive per-step but can require substantially fewer iterations.
Recent work has shown that Gauss-Newton substantially outperforms Adam even in language model training~\citep{abreu2025potential}.
This raises a natural question for seeking the best of both worlds: is the more expensive curvature information always better than gradient-based statistics, and if not, under what conditions?

This question has not been cleanly isolated in a controlled comparison
despite the large amount of research on Adam~\citep{zhang2020adaptive,kunstner2023sign,das2024towards} and Gauss-Newton~\citep{martens2010deep,pascanu2013revisiting,botev17GN,gargiani2020promise},
partly due to a lack of clear depiction of the design space.
In this work, we disentangle several factors in preconditioner design.
We consider a preconditioner in its factorized form $\precond= \ebasis \precondD^\power \ebasis^\top$,
where $\precondD$ is the preconditioner applied with power $\power$,
and $\ebasis$ is the basis that $\precondD$ operates under.
This factorization makes it easy to pinpoint preconditioner differences.
For example, Shampoo~\citep{gupta2018shampoo} and SOAP~\citep{vyas2024soap} share the same basis (i.e. Shampoo's eigenbasis) but differ in the diagonal scaling;
Muon~\citep{jordan2024muon} reduces to Shampoo when turning off the momentum in basis estimation.
SOAP also differs from Adam in the basis,
and differs from GN as studied in \cite{abreu2025potential} in both the diagonal scaling and the basis choice: the latter uses the full unfactorized eigenbasis of the GN matrix, where SOAP (and Shampoo) constructs its basis using a computationally efficient Kronecker factorization.

Notably, the Adam optimizer and GN methods differ in all of $\ebasis$, $\precondD$, and $\power$.
Adam operates under the identity basis ($\ebasis=I$), and uses gradient statistics for the diagonal scaling (the \textbf{Adam-style} scaling, Equation \ref{eq:diag_adam}) with power $\power=-\frac{1}{2}$ (i.e., square-root of the squared gradients).
In contrast, the Gauss-Newton method operates under the eigenbasis of the GN matrix (Equation \ref{eq:gn_def}),
and takes the diagonal of the GN matrix (the \textbf{GN-style} scaling, Equation \ref{eq:diag_gn}) with power $\power=-1$.

Our results highlight the importance of separating the choice of the \textit{diagonal scaling} $\precondD$ (and $\power$) from the quality of the \textit{basis estimation} $\ebasis$ under which $\precondD$ operates,
where quality refers to the alignment with the curvature as captured by the GN eigenbasis.
We further examine these choices under different \textit{gradient noise} settings,
namely the population and stochastic regimes,
which strongly influence the ranking.
We demonstrate this with linear regression.
With population gradients, Adam-style scaling exhibits an \textit{auto-tuning effect} that adapts to the true curvature even with a poor basis choice,
allowing it to outperform GN-style scalings.
In contrast, in the stochastic regime, Adam-style behaves similarly to GN$^{-1/2}$ regardless of the basis choice.
This reveals a connection between empirical and true Fisher:
with single-sample gradients,
Adam-style scaling corresponds to the diagonal of the \textit{empirical} Fisher~\citep{KunstnerHB19}, whereas the GN matrix reflects curvature information of the \textit{true} Fisher or Hessian.
Intuitively, the difference between regimes arises because the Adam-style scaling decomposes into contributions from the squared mean gradient and the gradient variance, with their relative importance depending on the gradient-noise regime.
\footnote{
Note that the GN matrix is always treated and estimated as a population quantity regardless of the gradient noise, hence the batch size only affects gradient noises and not basis estimation quality.
}

Beyond the effects of basis choice and gradient noise,
we show that Adam-style scaling can outperform GN-style scaling even when using population updates under the eigenbasis.
We demonstrate this on a non-convex variant of logistic regression,
where Adam-style effectively reduces to sign gradient descent and achieves dimension-independent convergence.
In contrast, for $\GN^{-1}$, ensuring global stability requires the learning rate to be sufficiently small, which hinders its local convergence and results in a polynomial dimension dependence.
This identifies a setting in which gradient statistics (as used in Adam and SOAP) can be more favorable than curvature information,
rather than merely serving as a tractable proxy.

\begin{table}[t]
  \centering
  \small{
  \setlength{\tabcolsep}{10pt}  %
  \renewcommand{\arraystretch}{1.25}  %
    \caption{Comparing Adam-style (Equation \ref{eq:diag_adam}) vs GN-style (Equation \ref{eq:diag_gn}) diagonal scaling
    across (i) \textit{basis choices}, i.e., the
    eigenbasis or identity basis.
    and (ii) \textit{gradient noise regimes}, i.e., population or stochastic.
    Our theoretical results are based on quadratics (\S\ref{sec:full_batch_incorrect}, \S\ref{sec:stochastic})
    and a non-convex variant of logistic regression (\S\ref{sec:logistic}).
    We use $A > B$ to denote $A$ converging faster than $B$.
    }
    \label{tab:basis-batch-grid}
   \begin{tabular}{@{}>{\raggedright\arraybackslash}p{2cm}@{}>{\centering\arraybackslash}p{6.5cm}>{\centering\arraybackslash}p{6cm}@{}}
    \toprule
    & \multicolumn{2}{c}{\textbf{Gradient noise}} \\
    \cmidrule(lr){2-3}
    \textbf{Basis Choice} & Population regime & Stochastic regime (\S\ref{sec:stochastic}) \\
    \midrule
    Identity basis &
     $\exists$ quadratic example s.t. Adam $> \GN^{-1}$ (\S\ref{sec:full_batch_incorrect}) &
      $\GN^{-1/2} \approx$ Adam for quadratics
    \\
    Eigenbasis &
    $\exists$ logistic example s.t. Adam $> \GN^{-1}$ (\S\ref{sec:logistic}) &
    $\GN^{-1} \geq \GN^{-1/2} \approx$ Adam for quadratics
    \\
    \bottomrule
  \end{tabular}
  }
  \vspace{-1.5em}
\end{table}

To summarize, 
we compare Adam-style vs GN-style diagonal scaling under the interaction of two factors:
\textbf{basis choice} regarding alignment with the true curvature, and \textbf{gradient noise} from mini-batching.
Our results are summarized in \Cref{tab:basis-batch-grid}, with the following main contributions:
\vspace{-0.6em}
\begin{itemize}
    \item \textbf{Separating diagonal effects from basis choices and gradient noises.} 
    We formulate Adam-style and GN-style diagonal preconditioners in a common framework that disentangles the diagonal scaling from its operating basis, in particular the curvature-aligned eigenbasis and the uninformative identity basis, as well as the gradient noise regime.
    This separation proves to be necessary, as the comparison is sensitive to the change along either axis (\Cref{tab:basis-batch-grid}).

    \item \textbf{Effect of basis misalignment (\S\ref{sec:full_batch_incorrect}).}
    For quadratics (i.e., linear regression), we construct an instance where the GN-style scaling under the identity basis loses the relevant curvature information, while Adam-style normalization ``auto-tunes'' its effective step sizes and converges faster than both $\GN^{-1}$ and $\GN^{-1/2}$ (\S\ref{sec:autotune}).
    Moreover, neither $\GN^{-1}$ nor $\GN^{-1/2}$ strictly dominates the other when the basis is misaligned (\S\ref{sec:GN_power}).

    \item \textbf{An equivalence in the stochastic regime (\S\ref{sec:stochastic}).}
    Continuing the linear regression setup above but changing to single-sample gradients, we show that Adam-style scaling and $\GN^{-1/2}$ are equivalent up to loss-dependent constant factors, regardless of the basis choice.
    Basis alignment continues to matter in the stochastic regime: $\GN^{-1}$ is preferable over $\GN^{-1/2}$ when using the correct eigenbasis, but can be worse under the identity basis.
    
    \item \textbf{A separation even with population updates under aligned curvature (\S\ref{sec:logistic}).}
    Outside linear regression, $\GN^{-1}$ can be suboptimal even with perfect basis alignment and no gradient noise:
    for a non-convex variant of logistic regression, global stability requirements leads to a polynomial slowdown in the local convergence of $\GN^{-1}$ despite working in its eigenbasis, where Adam-style scaling avoids this slowdown.
\end{itemize}

We complement these results with experiments on synthetic problems, CIFAR10, and Transformers (\S\ref{sec:experiments}).
\footnote{Code can be found at \url{https://github.com/ClaraBing/Adam_or_GN/}.}
The experiments reproduce the qualitative phenomena predicted by results in \Cref{tab:basis-batch-grid} and show that their implications carry to more realistic settings.
We discuss related work in \S\ref{sec:related_work}.

\subsection{Related work}
\label{sec:related_work}

\paragraph{Approximate second-order optimizers}
Recent advances on approximate second-order methods have demonstrated success in large-scale settings, serving as efficient alternatives to classic second-order methods such as Newton's and natural gradient descent, which are computationally bottle-necked to scale to high dimensions.
While first-order diagonal preconditioners are shown to be comparable in practice~\citep{kaddour23gain,zhao2024deconstructing}, leveraging second-order information has proven effective.
Most relevant to our work are methods that can be considered as applying diagonal preconditioners in a chosen basis~\citep{gupta2018shampoo,liu2023sophia,vyas2024soap,jordan2024muon}.
However, there is no clear understanding of how the preconditioner and the basis interact.
For instance, Sophia~\citep{liu2023sophia} can be considered as applying $\GN^{-1}$ in the identity basis, which as we will show, is not always desired.
In contrast, our work provides a clarifying decomposition of the design space of these second-order methods by separating the choice of basis in which to perform a diagonal preconditioner, from the choice of the diagonal preconditioner itself. 
For the stochastic regime, \cite{martens2014new} provided convergence results with a dependence on the condition number, similar to \Cref{lem:precond_loss_rate}.
However, their setting crucially differs from ours by assuming that the covariance of the gradients is independent of the current iterate.

Regarding the comparison of Adam and Gauss-Newton, a closely related work is \cite{abreu2025potential} which provides an empirical study demonstrating that for language model pretraining, Gauss-Newton requires fewer steps to converge than common optimizers such as Adam, SOAP~\citep{vyas2024soap} and Muon~\citep{jordan2024muon}.
Our results are complementary to theirs and focus more on theoretical understanding.
An important difference to note is that in \cite{abreu2025potential}, \GN~is performed with per-step line search for the step size.
In contrast, our result in \S\ref{sec:logistic} is stated for a non-increasing learning rate schedule (as commonly used in practice) and hence does not contradict the empirical findings in \cite{abreu2025potential}.

\paragraph{Efficient optimizers for large-scale training}
Although preconditioned methods are theoretically appealing for their faster convergence, substantial efforts have been made to translate these gains to practical speedups in wall-clock time, which is crucial in modern large-scale training.
To keep each update step lightweight, it is common to approximate the Hessian using the Gauss-Newton matrix (\Cref{eq:gn_def}), which, despite being biased, relies only on gradient information and is therefore more computationally efficient than the other commonly used Hutchinson estimator.
In addition, when estimating the eigenbasis, one can use the Kronecker factorization in place of the full basis~\citep{martens15KFAC,george2018fast,vyas2024soap}.
In this work, we analyze the full eigenbasis of the Gauss-Newton matrix,
while adopting the Kronecker approximation in the experiments (see \S\ref{app:expr_details} for details).

\paragraph{Adam vs (S)GD}
There have been a lot of interest understanding the comparison of Adam and (stochastic) gradient descent.
Related to the preconditioning perspective in our work, \cite{das2024towards} studies Adam's preconditioning effect on quadratics, and shows that it outperforms SGD when the Hessian is sufficiently ill-conditioned.
A line of work focuses the comparison on optimizing Transformers,
investigating it through the lens of gradient noises~\citep{zhang2020adaptive}, relation to sign descent~\citep{kunstner2023sign}, and the curvature of the landscape~\citep{jiang2023geometry,pan2023toward};
\cite{Ahn24linear} provides a review and a theory-friendly abstraction.
Most related to our work is~\cite{maes2024understanding}, which show that Adam's advantage over SGD for Transformers relies crucially on the choice of basis.
While these results hinge on properties specific to Transformers, we are interested in understanding algorithm design with insights that can be generally applicable.

\section{Preliminaries}
\label{sec:setup}

Consider optimizing a function $f: \R^{\dim} \rightarrow \R$ parameterized by the (vectorized) parameter $\param \in \R^{\dim}$ against a loss function $\loss$.
Updates performed by preconditioning optimizers can be seen as
$$\param[t+1] = \param[t] - \lr \cdot (\ebasis \precondD^{\power} \ebasis^\top) \grad[t],$$
where $\lr$ denotes the learning rate,
$\precondD$ is the diagonal preconditioner which is raised to the exponent $\power \in \{-\frac{1}{2}, -1\}$, 
$\ebasis$ is the orthonormal basis on which the preconditioned update is performed, 
and $\grad[t]$ denotes the gradient at time $t$ (which could correspond to either population gradient or stochastic gradient depending on the setting).
The full update is described in~\Cref{alg:precond_opt},
and we discuss the choice of the basis and the diagonal preconditioner below.

We start with describing the \textit{Gauss-Newton} (GN) matrix, which is the first term in the following decomposition of the Hessian:
\begin{equation}
\label{eq:gn_def}
    \hessian := \nabla^2_{\param} \loss = \nabla_{\param} f \nabla^2_f \loss \nabla_{\param} f^\top + \nabla_f \loss \nabla^2_{\param} f := \hGN + \nabla_f \loss \nabla^2_{\param} f.
\end{equation}
Contrast to the Hessian $\hessian$, the GN matrix $\hGN$ requires only first-order information of the network to compute and often serves as a reasonable preconditioner in practice~\citep{SankarKVB21}.
For convex loss functions, which will be the focus of this work, the GN term is positive-semidefinite (PSD) and admits a real-valued eigendecomposition.

\textbf{Basis estimation.}
We focus on comparing two basis choices: 1) the identity basis $\ebasis = I$, and 2) the eigenbasis of the Gauss-Newton matrix $\hGN$.
For the experiments, we additionally consider the Kronecker-factored preconditioner~\citep{martens15KFAC,vyas2024soap} as a computationally efficient approximation of the eigenbasis: for a matrix-valued parameter $\Param \in \R^{\dimh \times \dimw}$, $\hGN \in \R^{\dimh\dimw \times \dimh\dimw}$ is approximated by the outer product of two matrices of dimension $\R^{\dimh \times \dimh}$ and $\R^{\dimw \times \dimw}$;
see \S\ref{app:kron} for details.

\textbf{Diagonal preconditioners.}
Given an orthonormal basis $\ebasis$, we first rotate the gradient $\grad$ into the basis $\gradrot := \ebasis^\top \grad$, then apply a diagonal conditioner $\precondD$ to the rotated gradient.
We are interested in two types of $\precondD$: Adam and Gauss-Newton.

For Adam, we use the following definition for theoretical analyses, where the preconditioner has diagonal entries
\begin{equation}
\label{eq:diag_adam}
    \precondD[A]_{ii}
    := \big(\E[(\gradrot(x)_{i})^2]\big)^{-1/2}
    = \big(\E[(\evec_i^\top \grad(x))^2]\big)^{-1/2},
\end{equation}
where the expectation is over all possible batches.
Note that in the full batch case, this simply corresponds to the rotated gradient,
while in the stochastic case, it depends on the norm of the per-sample rotated gradient.
This is motivated from the practical version of Adam, which maintains a running average of the gradients seen during the training.

For Gauss-Newton, it takes an additional exponent parameter $\power \in \{-\frac{1}{2}, -1\}$ and computes the diagonal elements as
\begin{equation}
\label{eq:diag_gn}
    \precondD[\GN]_{ii}
    := (\evec_i^\top \hGN \evec_i)^{\power},
\end{equation}
where $\evec_i$ is the $i^{\text{th}}$ vector in the given basis.
In particular, when $\ebasis$ is the eigenbasis of $\hGN$, $\{\evec_i^\top \hGN \evec_i\}_{i \in [\dim]}$ give the eigenvalues of $\hGN$.

The preconditioners for Adam and Gauss-Newton at batch size $1$ for standard losses like cross-entropy (CE) and Mean Squared Error (MSE) correspond respectively to \textit{empirical Fisher matrix} and the \textit{Fisher matrix}.
The former is defined with gradients with respect to labels from the true data distribution, whereas the latter is defined with respect to the output of the model.

\begin{algorithm}[t]
  \centering
  \caption{Preconditioned optimizer}\label{alg:precond_opt}
  \begin{tabular}{@{}r p{0.9\textwidth}@{}}
    \alglinenumber\ & \textbf{Input:} %
    $\param[0]$,
    \textsc{BasisType} $\in \{\text{Id, EigenBasis}\}$,
    \textsc{precond} $\in \{\text{Adam}, \GN\}$,
    power $\power \in \{-0.5, -1\}$,
    learning rate $\{\eta_t\}_{t=1}^T$,
    regularization coefficient $\epsilon$,
    gradient batch size $\bt$, and basis estimation batch size $\btH$.
    \\ 
    \alglinenumber\ & \textbf{for} $t=1$ \textbf{to} $T$ \textbf{do} \\
    \alglinenumber\ & \quad Sample a batch of $\bt$ samples $X_G := \{x_i\}_{i \in [\bt]}$. \\
    \alglinenumber\ & \quad Compute batch loss $\loss_t(\param[t]; X_G)$ and gradient $\grad[t] = \nabla \loss_t(\param[t]; X_G)$. \\
    \alglinenumber\ & \quad Sample a batch of $\btH$ samples $X_H := \{x_i\}_{i \in [\btH]}$. \\
    \alglinenumber\ & \quad Compute the Gauss-Newton matrix $\hGN$ using $X_H$.\\
    \alglinenumber\ & \quad \textbf{if} \textsc{BasisType} == Id \textbf{then} \textcolor{gray}{\texttt{// Basis choice}} \\
    \alglinenumber\ & \quad \quad Basis $\ebasis \gets I$. \\
    \alglinenumber\ & \quad \textbf{elif} \textsc{BasisType} == EigenBasis \textbf{then} \\
    \alglinenumber\ & \quad \quad Basis $\ebasis \gets \text{EigenDecomposition}(\hGN)$. \\ %
    \alglinenumber\ & \quad Compute the basis-rotated gradient $\gradrot[t] = \ebasis^\top \grad[t]$. \\
    \alglinenumber\ & \quad \textbf{if} \textsc{precond} == Adam \textbf{then} \textcolor{gray}{\texttt{// Diagonal preconditioner choice}} \\
    \alglinenumber\ & \quad \quad Compute the diagonal preconditioner as $\precondD_{ii} = \big(\E[(\gradrot(x)_{i})^2]\big)^{-1/2}$. \\
    \alglinenumber\ & \quad \textbf{elif} \textsc{precond} == GN \textbf{then} \\
    \alglinenumber\ & \quad \quad Compute the diagonal preconditioner as $\precondD_{ii} = (u_i^\top \hGN u_i)^{\power}$. \\
    \alglinenumber\ & \quad $\param[t+1] =  \param_t - \ebasis \precondD\gradrot[t] = \param_t - \ebasis \precondD\ebasis^\top \grad[t].$ \\
  \end{tabular}
\end{algorithm}

\section{Separation under an uninformative curvature}
\label{sec:full_batch_incorrect}

We start with analyzing linear regression, i.e., minimizing $\frac{1}{2}\E[((\param - \param^*)^\top x)^2]$, where $\GN^{-1}$ is known to achieve optimal convergence when using population gradients under the correct eigenbasis.
This section shows that the optimality $\GN^{-1}$ is sensitive to the basis choice:
when using the population gradients but the \textit{incorrect} identity basis, $\GN^{-1}$ can be outperformed by both the Adam-style scaling and $\GN^{-1/2}$.
In \S\ref{sec:autotune}, we construct an example when the incorrect basis fails to reflect the true curvature, making $\GN^{-1}$ severely suboptimal, while Adam remains adaptive with an ``auto-tuning'' effect.
We then show in \S\ref{sec:GN_power} that the comparison of $\GN^{-1}$ and $\GN^{-1/2}$ depends on the data covariance $\Cov_x$, and numerically identify choices of $\Cov_x$ for which either method can converge faster.

\subsection{Adam auto-tunes to the curvature}
\label{sec:autotune}

This section compares Adam-style scaling with $\GN^{-1}$.
We provide an example where the wrong basis obscures the true curvature, voiding GN's adaptiveness.
Recall that the gradient for quadratics is $\grad = \Cov_x \DeltaParam$, where $\DeltaParam := \param - \param^*$, and $\Cov_x$ is the input covariance.
Consider an input covariance
\begin{equation}
\label{eq:sparse_example}
\Cov_x := \E[xx^\top] = \begin{bmatrix}
   \mathbf{1}\mathbf{1}^\top & \mathbf{0} \\
    \mathbf{0} & \Id
\end{bmatrix} \in \R^{2\dim \times 2\dim},
\end{equation}
where $\mathbf{1} \in \R^\dim$ is the all-one vector, and $\Id \in \R^{\dim \times \dim}$ is the identity matrix.
This problem has a block structure that is symmetric among the first $\dim$ coordinates and among the last $\dim$ coordinates.
The two blocks have widely different maximal eigenvalues and hence different optimal learning rates:
The first block has a maximum eigenvalue of $\dim$, and thus the maximum stable learning rate is the $\frac{2}{\dim}$.
In contrast, all eigenvalues for the second block are 1, hence can afford a maximum learning rate of 2.
The wrong basis choice we consider is the identity basis, which hides the true curvature of the problem.
We will see that this makes GN fail to adapt to the curvature of the problem, while Adam remains efficient via an ``auto-tuning'' effect.

\textbf{GN converges slowly.}
When taking $\ebasis = I$, the diagonal preconditioner for \GN~ has $\precondD[\GN]_{ii} = 1^{\power} = 1$ (recall that $\power$ is the exponent for \GN).
That is, both $\GN^{-1}$ and $\GN^{-\frac{1}{2}}$ simply scale all coordinates by the same factor as all diagonal elements are 1, behaving the same as vanilla gradient descent.

\textbf{Adam ``auto-tunes'' to the curvature.}
Adam effectively acts as normalized gradient descent for the first block since the per-coordinate gradients are the same, and acts as signed gradient descent in each coordinate in the second block.
Let $|\grad[t]_{0}|$ denote the norm of the coordinates of the first block gradient, and $|\grad[t]_i|$ for $i \in [\dim]$ denote the per-coordinate gradient norm for the coordinates in the second block which evolves independently.

With gradient $\grad = \Cov_x \DeltaParam$, the update has $\DeltaParam[t+1] = (\Id - \lr \Cov_x) \DeltaParam[t]$ for vanilla gradient descent,
and the gradient norm goes down as long as $\lr \leq \frac{2}{\eval_{\max}}$.
For Adam, the updates can be considered as gradient descent with an adaptive learning rate.
Specifically, following \Cref{eq:diag_adam}, the first block updates as $\DeltaParam[t+1]_0 = (\Id - \frac{\lr}{\|\grad[t]_0\|_2} \cdot \Cov_x) \DeltaParam[t]_0$,
and the second block has per-coordinate updates $\DeltaParam[t+1]_i = (\Id - \frac{\lr}{|\grad[t]_i|_2} \cdot \Cov_x) \DeltaParam[t]_i$ for $i \in [\dim]$.
This means:
\vspace{-0.8em}
\begin{enumerate}[leftmargin=*]
    \item $|\grad[t]_0|$ decreases provided $\lr/|\grad[t]_0| \leq 2/\dim$,
    \quad
    2. $|\grad[t]_i|$ decreases for $i > 0$ provided $\eta / |\grad[t]_i| \leq 2$.
\end{enumerate}
Thus, after an initial ``burn-in'' period, $\lr / \|\grad[t]_0\|$ reaches $2/\dim$ and oscillates around this value,
while $\lr / |\grad[t]_i|$ oscillates around 2.
\footnote{
More precisely, for any $t \geq t_0 := \max\{\frac{|\grad[0]_0|}{\dim\lr}, \frac{|\grad[0]_i|}{\lr}\}$,
we have that 
1) $|\grad[t+1]_0 - \grad[t]_0| = \dim\eta$ and $|\grad[t+1]_0| + |\grad[t]_0| = \dim\eta$ for the first block,
2) $|\grad[t+1]_i - \grad[t]_i| = \eta$ and $|\grad[t+1]_i| + |\grad[t]_i| = \eta$, $\forall i\in[\dim]$ for the second block.
}
We refer to this as the \textit{auto-tuning} of Adam: it adapts to the curvature of different coordinates on its own, by regulating the gradient norms.
After this burn-in period, we can reduce the learning rate by half at every step, reaching a target error within log number of steps.

This \textit{auto-tuning} effect is similar to adapting to the smoothness of the curvature, which is a known property of normalized gradient descent~\citep{orabona2023normalizedgradients}.
Note that auto-tuning comes from the norm of the mean gradient in the Adam's denominator for the \textit{full batch} case.
Under small batches where the gradient variance dominates, Adam does not exhibit autotuning and instead behaves similarly to $\GN^{-1/2}$, as we will show in \S\ref{sec:stochastic}.
\footnote{Recall the definition of $\precondD[A]_{ii}$ from \Cref{eq:diag_adam}.
For the $i_{th}$ basis vector $\evec_i$, let $\mu_i := \evec_i^\top \E_x[\grad(x)]$ and $\sigma_i^2 := \E[(\evec_i^\top \grad(x) - \mu_i)^2]$ denote the mean and variance of the gradient projection along $\evec_i$.
Consider gradients computed on a batch of size $B$, then 
$(\precondD[A]_{ii})^2 = \mu_i^2 + \frac{1}{B}\sigma_i^2$.
$(\precondD[A]_{ii})^2$ is dominated by the mean for population updates ($B \rightarrow \infty$),
and is often dominated by the variance in the stochastic regime ($B=1$).
}
Concurrent work by \cite{roulet2025perexample} provides consistent empirical evidence that at small batch sizes, keeping only the mean term in the Adam's denominator improves its performance, thus supporting the auto-tuning viewpoint of Adam.

\subsection{Which power to use for Gauss-Newton?}
\label{sec:GN_power}
\vspace{-0.6em}

We next compare two variants of the GN-style scaling, $\GN^{-1}$ and $\GN^{-1/2}$.
Although the Newton's method was originally proposed with a preconditioner closer to $\GN^{-1}$, the square root in the denominator of Adam \citep{kingma2014adam} and Adagrad \citep{duchi11adagrad} has spurred multiple papers questioning the correct power of the preconditioner to be used in practice~\citep{liu2023sophia,lin24remove,vyas2024soap}, motivating us to consider $\GN^{-1/2}$.
We show that there exists problems for which $\GN^{-1}$ fails to dominate $\GN^{-1/2}$ under the incorrect identity basis.
This is because their convergence rates
depend on the condition number of the preconditioned Hessian~\citep{Boyd_Vandenberghe_2004}, i.e., $\kappa(\precond^{1/2} \Cov_x \precond^{1/2})$ for a general preconditioner $\precond$,
and there exists $\Cov_x$ for which taking $\precond$ to be $\GN^{-1/2}$ results in a smaller condition number than $\GN^{-1}$.
We provide details in \S\ref{app:gn_power} and accompanying simulation results in \S\ref{sec:simulation} (\Cref{fig:GN-powers}).

\section{Stochastic regime: Adam and GN\texorpdfstring{$^{-1/2}$}{-1/2} behave similarly}
\label{sec:stochastic}
\vspace{-0.6em}

The previous section shows that when using population gradients, Adam and $\GN^{-1/2}$ can both outperform GN$^{-1}$ if the basis choice poorly reflects the true curvature.
In this section, we continue the quadratic linear regression setting, but focus instead on the stochastic regime (i.e. batch size 1).
We first establish that Adam-style and $\GN^{-\frac{1}{2}}$ behave similarly regardless of the basis choice,
and then compare $\GN^{-\frac{1}{2}}$ (and hence Adam) with $\GN^{-1}$, showing that no one strictly dominates the other.
Proofs for this section can be found in \S\ref{app:stochastic_regime}.

\vspace{-0.6em}
\paragraph{Adam and $\GN^{-1/2}$ are equivalent up to a loss-dependent constant.}
We first state a general result for Gaussian input distribution and quadratic loss, that the empirical Fisher $\E[\grad(x) \grad(x)^\top]$ is approximately equivalent to the true Fisher $\Cov_x$ up to a loss scaling. 
\begin{lemma}
\label{lem:adam_gn_half_fisher}
    For linear regression with Gaussian inputs, the following holds:
    \[
        \loss(\param) \cdot \Cov_x \preceq \frac{1}{2} \E[\grad(x) \grad(x)^\top] \preceq 3\loss(\param) \cdot \Cov_x.
    \]
\end{lemma}
\vspace{-1em}

Utilizing \Cref{lem:adam_gn_half_fisher}, we show that $\GN^{-\frac{1}{2}}$ and Adam behave the same upto a scalar constant in any basis under the stochastic regime.
\begin{corollary}
\label{lem:adam_gn_half_update}
    For single-sample updates, the updates of Adam and $\GN^{-\frac{1}{2}}$ differ by a loss-dependent constant:
    $\frac{1}{\sqrt{3\loss}} \cdot \precondD[\GN, -\frac{1}{2}] \preceq \sqrt{2} \precondD[A] \preceq \frac{1}{\sqrt{\loss}} \cdot \precondD[\GN, -\frac{1}{2}].$
        
\end{corollary}
The similarity between Adam-style and $\GN^{-1/2}$ scaling is empirically verified in \S\ref{sec:experiments}.

\paragraph{Comparing $\GN^{-1}$ and $\GN^{-1/2}$.}
We next turn to the comparison between $\GN^{-1/2}$ (and hence Adam) and $\GN^{-1}$.
To answer this, we first quantify the convergence rate of a general preconditioner $\precond$ for linear regression with stochastic Gaussian inputs.
Denoting the preconditioned Hessian as $\mA(\precond) := \precond^{1/2}\Cov_x \precond^{1/2}$,
we have that:
\begin{lemma}
\label{lem:precond_loss_rate}
    For linear regression with stochastic Gaussian inputs,
    it holds for a general $\precond$ that
    \[ \E[\loss[t]] \leq O\left[\left(1 - \frac{\lambda_{\min}(\mA(\precond))}{3\Tr(\mA(\precond))}\right)^t \loss[0]\right]. \]
\end{lemma}

\Cref{lem:precond_loss_rate} shows that the convergence bound is governed by a condition-number-like quantity $\kappa_s(\mA(\precond)) = \frac{\Tr(\mA(\precond))}{\lambda_{\min}(\mA(\precond))} = \sum_{i \in \dim} \frac{\lambda_i(\mA(\precond))}{\lambda_{\min}(\mA(\precond))}$, similar to the condition-number dependency discussed in \S\ref{sec:GN_power}.
Note that $\kappa_s \geq \dim$.
There are two implications.
First, in the correct basis, $\GN^{-1}$ optimizes this convergence guarantee
even in the stochastic regime, as it minimizes the condition number to $\kappa_s(\mA(\precond)) = \kappa_s(\Id) = \dim$;
in comparison, $\GN^{-1/2}$ has $\kappa_s(\mA(\precond)) = \kappa_s((\Cov_x)^{1/2}) = \sum_i \sqrt{\frac{\lambda_i(\Cov_x)}{\lambda_{\min}(\Cov_x)}}$, which is greater than $\dim$ unless $\lambda_{\max}(\Cov_x) = \lambda_{\min}(\Cov_x)$.
When using an incorrect basis though, similar to \S\ref{sec:GN_power},
there exists $\Cov_x$ for which $\GN^{-1/2}$ leads to a smaller $\kappa_s$ than $\GN^{-1}$.
We numerically demonstrate such a $\Cov_x$ choice in \Cref{fig:GN-powers}.

\section{Beyond quadratics: \texorpdfstring{$\GN^{-1}$}{GN} can be suboptimal under ideal setup}
\label{sec:logistic}

The previous sections discussed how poor basis estimation and gradient noise can affect the effectiveness of GN-style diagonal scaling.
In this section, we show that in more general setups, $\GN^{-1}$ can be provably suboptimal even when using population gradients under the perfect basis.
Specifically, for the local convergence of a non-convex variant of logistic regression,
Adam-style scaling enjoys dimension-free convergence, whereas $\GN^{-1}$ suffers from a polynomial-in-dimension slowdown.

\textbf{Setup.}
The input $x$ is a random $\dim$-dimensional one-hot vector sampled from $\{e_i\}_{i=1}^{\dim}$ with probability $\nu_i := \Pr(x=e_i)$.
Conditional on $x=e_i$, the label $y$ is Bernoulli with mean $\py_i := \Pr(y=1|x=e_i)$.
We assume $0.6 \leq \py_i \leq 0.8$, $\forall i\in[d]$, i.e. the labels are neither deterministic nor fully random, and the optimal parameter has a bounded norm.

The objective is $-\sum_i \nu_i\big(\py_i \log \pf_i+ (1-\py_i) \log(1-\pf_i)\big)$,
where $\pf_i$ denotes the model $\pf$'s output on $e_i$.
We take $\pf: \R^{\dim} \rightarrow \R^{\dim}$ to be a weight-tied \emph{two-layer linear network},
whose output depends on the \emph{squares} of the weights:
for any $\param\in\R^{\dim}$ and for $i \in [\dim]$, we define the model's prediction as 
\begin{equation}
    \pf_i(\param) = \Pr_{\param}(y=1\mid x=e_i)
        = \sigma\!\Bigl(\sum_{j=1}^{\dim} \param_j^{2}x_j\Bigr)
        = \sigma(\param_i^{2}),
    \qquad
    \sigma(z)=\frac1{1+\exp(-z)}.
\end{equation}
The square parameterization makes the problem non-convex, and is analogous to the structure of key-query multiplication in self-attention~\citep{vaswani2017attention}.

In the following, we show that under the natural assumption of non-increasing step sizes, there is a separation between Adam and $\GN^{-1}$ in terms of $\kappa(\nu) := \frac{\nu_{\max}}{\nu_{\min}}$.
We consider a choice of $\nu$ such that $\kappa(\nu) = \Omega(\dim^{1/2+\delta})$ for some $\delta \in [0, \frac{1}{2}]$.
We are interested in convergence in parameter distance: let $\epsilon$ denote the target parameter error, we want to find $\param$ such that $\|\param - \param^*\|_2 \leq \epsilon$.
Here $\param^* := \min_{\tilde\param \in \Theta^*}\|\param - \tilde\param\|_2$,
and $\Theta^*$ denotes the set of minimizers which is of size $2^\dim$ due to sign flips.

\textbf{Adam-style converges in $O(\log(1/\epsilon))$ steps.}
Adam-style effectively reduces to sign GD, hence the amount of parameter update is determined by the step size.
The $O(\log(1/\epsilon)$ convergence hence follows directly from starting with a $O(1)$ learning rate and halving the step size every $O(1)$ steps.

\textbf{$\GN^{-1}$ requires $\tilde{\Omega}(\dim^\delta\log(1/\epsilon))$ steps.}
In contrast to Adam, GN's precondtioning can result in an unboundedly large update that makes optimization diverge 
\footnote{Recall that the optimum is bounded given a $\py$ bounded away from 1, which differs from analyses on linearly separable data where the optimum is at infinity~\citep{wu2024large}.}
(see \Cref{eq:gn_update_logistic}),
unless the learning rate is kept small,
which in turn leads to slow convergence.
To show the lower bound, we first state a more general convergence result.
Recall that the $\GN^{-1}$ update is $\param[t+1] = \param[t] - \lr[t] (\hGN(\param)+\reg I)^{-1} \grad[t].$
We assume that $\{\lr[t]\}_{t \geq 0}$ are non-increasing and that the step size schedule and regularization lead to convergence, i.e. $\param[t]\!\to \param^*$ as $t\to \infty$. 
Linearizing the update map at the limit $\param^*$ gives 
\begin{equation}
    \param[t+1] - \param^*
    =\Bigl(I-\lr^{\infty}(\hGN_*+\reg I)^{-1} \hGN_*\Bigr)
     (\param[t]- \param^*)
     \;+\;O\bigl(\|\param[t]- \param^*\|^{2}\bigr),
\end{equation}
where $\hGN_* := \hGN(\param^*)$ and $\lr^\infty = \lim_{t\to\infty} \lr_t$.
We are interested in lower bounding the spectral radius of the local iteration matrix:
\[
    \spectralRad(\lr[\infty],\reg)
      :=\bigl\|
          I-\lr[\infty](\hGN_* + \reg I)^{-1} \hGN_*
       \bigr\|_{2},
\]
as it governs the ultimate \emph{local} rate of convergence that any $\GN$ schedule can achieve. We refer to $\spectralRad(\lr[\infty],\reg)$ as the local contraction factor; a value of $\spectralRad$ close to 1 implies slow convergence.

The main result of this section is a lower bound on $\spectralRad$:
\begin{theorem}
\label{theorem:lower_logistic}
Suppose the weights are initialized at $\param[0] = \frac{1}{\sqrt{\dim}}\cdot \vec{1}$.
Consider any non-increasing step size sequence $\{\lr[t]\}_{t \ge 0}$
and regularization parameter $\reg \geq 0$.
If the Gauss-Newton iterates converge to $\param^*$, i.e. $\param[t] \to \param^*$, then the local contraction
factor $\spectralRad(\lr[\infty], \reg)$ is lower bounded by: 
\[
\spectralRad(\lr[\infty],\reg)
\ge 1- c \sqrt{\log \dim}\,  \max\left\{\frac{1}{\sqrt{\dim}}, \sqrt{\dim}/\kappa(\nu)\right\}, 
\]
for some universal constant $c$.
\end{theorem}
\vspace{-0.6em}

This theorem reveals a basic trade-off: for the Gauss-Newton method to converge globally from our chosen starting point, its final learning rate must be small. This restriction, in turn, hampers its local convergence speed and creates a bottleneck.
The slowdown is substantial under the common conditions of high dimensionality and ill-conditioned data:
\begin{cor}
    Under the limit $\param[t] \to \param^*$,
    for imbalanced input with $\kappa(\nu) = \Omega(\dim^{1/2+\delta})$ for some $\delta \in [0, 1/2]$,
    $\GN^{-1}$ requires $t = \tilde{\Omega}(\dim^{\delta}\log(1/\epsilon))$ steps to reach a parameter satisfying $\|\param[t] - \param^*\|_2 \leq \epsilon$.
\end{cor}
This demonstrates a polynomial slowdown in the dimension, highlighting a scenario where the theoretical power of Gauss-Newton is significantly degraded due to the practical requirement of global convergence.
We empirically verify this on Transformers in \S\ref{sec:expr_attn}.

\textit{Remark}: 
Ideally, in a purely local setting, one could choose $\reg=0$ and use a constant final stepsize $\eta^\infty$.
For instance, setting $\eta^{\infty}=1$ would make the iteration matrix zero, yielding $\gamma=0$ and superlinear convergence. In fact, any other constant $\eta^{\infty} \in (0,2)$ would similarly provide rapid linear convergence with a rate independent of the condition number.
However, \Cref{theorem:lower_logistic} shows this ideal scenario is not possible.
As the proof (\S\ref{sec:logistic_proof}) shows, the requirement of ensuring convergence from a specific, natural initialization forces the algorithm's final step size, $\lr[\infty]$, to be small.
This constraint directly degrades the local contraction factor, preventing the rapid convergence one might expect from a Newton-like method.
Further, we note that our result does not contradict with the fast convergence from line search, which does not fall under the non-increasing step size assumption.

\section{Experiments}
\label{sec:experiments}

We provide experimental evidence
for the theoretical results in \Cref{tab:basis-batch-grid},
including
(i) simulations for linear regression results in \S\ref{sec:full_batch_incorrect} and \S\ref{sec:stochastic},
(ii) extending to non-convex examples with MLP,
and (iii) logistic examples with Transformer for \S\ref{sec:logistic}.

\textbf{Experiment details.}
We use the mean square error as the objective function unless otherwise specified.
For numerical stability, we optionally regularize $\precondD$ to be $\precondD + \reg I$ for some small $\reg > 0$.
We sweep over the learning rate $\lr$ and the regularization coefficient $\reg$.
For Adam, we also sweep over the learning rate schedule (constant or step decay) and $\beta_2$; we fix $\beta_1 = 0$, similar to~\cite{das2024towards}.
We use different samples for estimating gradient and Gauss-Newton.
Due to computational considerations, our eigenbasis experiments also consider Kronecker approximation of the full eigenbasis.
We report the mean and standard error based on 10 seeds.
More details are provided in \S\ref{app:expr_details}.

\paragraph{Quadratics simulations}
\label{sec:simulation}
We provide simulations for results in \S\ref{sec:full_batch_incorrect} and \S\ref{sec:stochastic}. 
\vspace{-0.6em}
\begin{itemize}
    \item \textit{Comparing Adam and GN under full-batch updates.}
    We empirically verify the examples provided in the theory where Adam can outperform GN both under the identity basis and the eigenbasis (\Cref{fig:sparse-example}).
    For the identity basis, we consider a 100-dimensional linear regression task where the covariance matrix has a block-wise structure following the construction in \S\ref{sec:autotune}.
    Adam converges quickly due to the auto-tuning effect, whereas Gauss-Newton converge as slowly as vanilla gradient descent.
    For logistic regression (\S\ref{sec:logistic}), our results on a 2048-dimensional problem  (detailed in \S\ref{app:expr_details_logistic}) show that Adam converges faster even under the eigenbasis.

    \item \textit{Comparing powers of Gauss-Newton.}
    \S\ref{sec:GN_power} shows that the comparison of $\GN^{-1}$ and $\GN^{-1/2}$ amounts to comparing the condition number of a particular matrix, for both population and stochastic settings.
    We empirically verify the claim on a 5-dimensional linear regression problem,
    controlling the choice of the covariance $\Cov_x$.
    \Cref{fig:GN-powers} shows the simulation results in this case,
    where $\GN^{-1/2}$ converges faster than $\GN^{-1}$ when using both large and small batches.
    Details are provided in \S\ref{app:gn_power}.
\end{itemize}

\begin{figure}[t]
    \begin{minipage}{.48\textwidth}
        \centering
        \includegraphics[width=\textwidth]{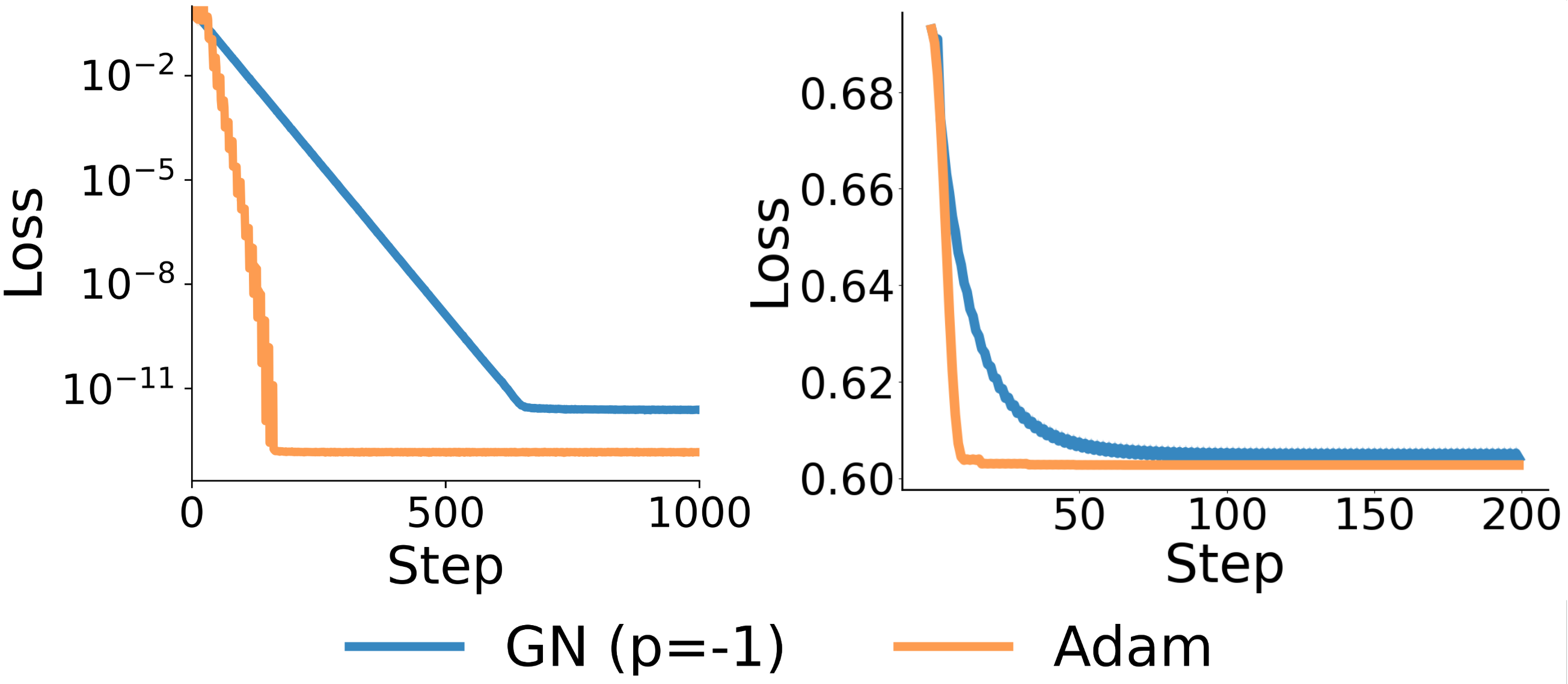}
        \caption{\textbf{Adam converges faster than GN with full batches}, (Left) under the identity basis on a linear regression task with block-wise covariance, where GN fails to adapt to the problem curvature;
        and (Right) under the eigenbasis, for the reparameterized logistic regression task.
        \label{fig:sparse-example}
        }
    \end{minipage}\hfill%
    \begin{minipage}{.47\textwidth}
        \centering
        \includegraphics[width=\textwidth]{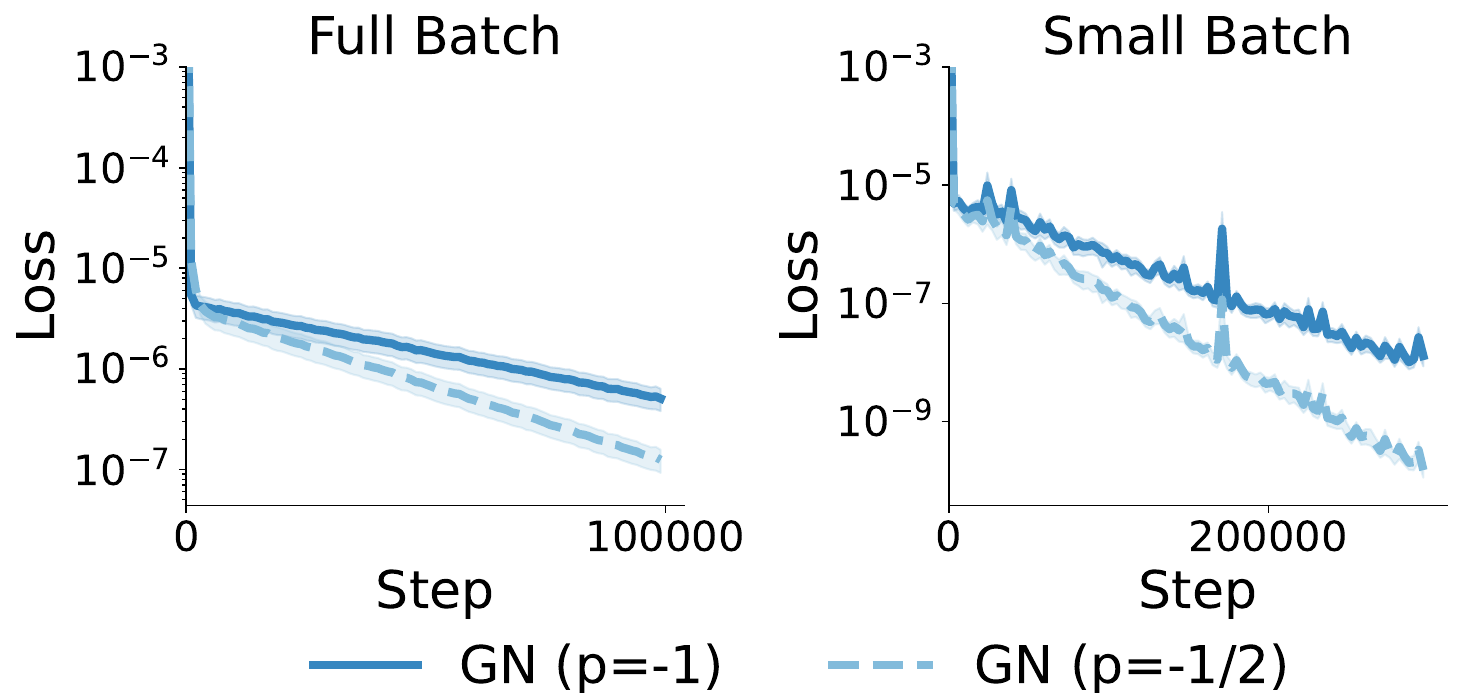}
        \caption{\textbf{Comparing GN power $\power \in \{-\frac{1}{2}, -1\}$}.
        On a regression task where GN$^{-1/2}$ leads to a more favorable condition number, GN$^{-1/2}$ converges faster than GN$^{-1}$ with both small and large batches.
        \label{fig:GN-powers}
        }
    \end{minipage}
\end{figure}

\begin{wrapfigure}[14]{r}{0.4\textwidth}
  \centering
  \vspace{-1.75em}
  \includegraphics[width=0.85\linewidth]{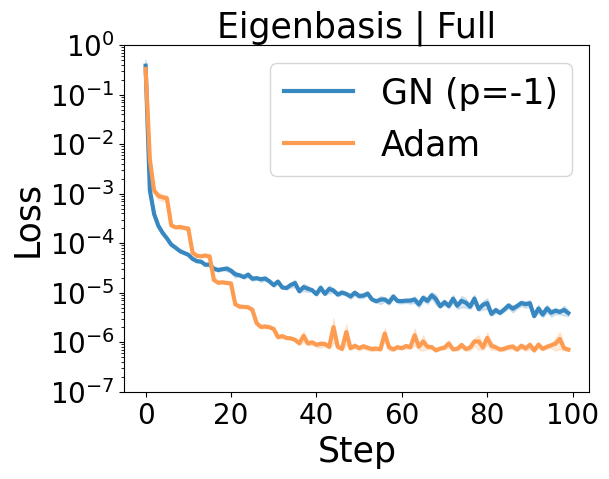}
  \caption{Transformer experiments (\S\ref{sec:expr_attn}):
  Adam outperforms $\GN^{-1}$ under the eigenbasis with full batches.}
  \label{fig:transformer}
  \vspace{-2em}
\end{wrapfigure}

\paragraph{Non-convex examples with MLP}
\label{sec:non_convex_mse}
Next, we consider non-convex optimization with one-hidden-layer MLPs on the following tasks:
1) learning from a random teacher network;
2) feature learning with sparse parity and its variant ``staircase'', where the labels depend on a subset of input coordinates;
and 3) CIFAR10 image classification, where the class labels are treated as one-hot vectors;
\S\ref{app:expr_details_mse} provides details.
Experiments on these tasks cover the full $2\times2$ grid in \Cref{tab:basis-batch-grid},
with respect to batch size (full vs small) and the basis (eigenbasis vs identity).
We use the Kronecker approximation for the eigenbasis of the Hessian, which behaves similarly to the full eigenbasis (\S\ref{app:kron}) while being more compute efficient.
We provide additional experiments on intermediate basis choices by interpolating between the identity and the eigenbasis in \S\ref{app:interpolation}.

As shown in Figures~\ref{fig:teacher-2x2-plots}--\ref{fig:cifar},
our theoretical analyses empirically extend to these non-convex tasks across both axes of interest.
In particular, across all problems considered, Adam and $\GN^{-1/2}$ closely track one another at small batch sizes, regardless of the basis chosen (subfigures (c), (d)).
Moreover, Adam is close to or better than GN$^{-1}$ when the basis is incorrect (subfigures (a), (c)).

\begin{figure}[t]
  \centering
  \includegraphics[width=\textwidth]{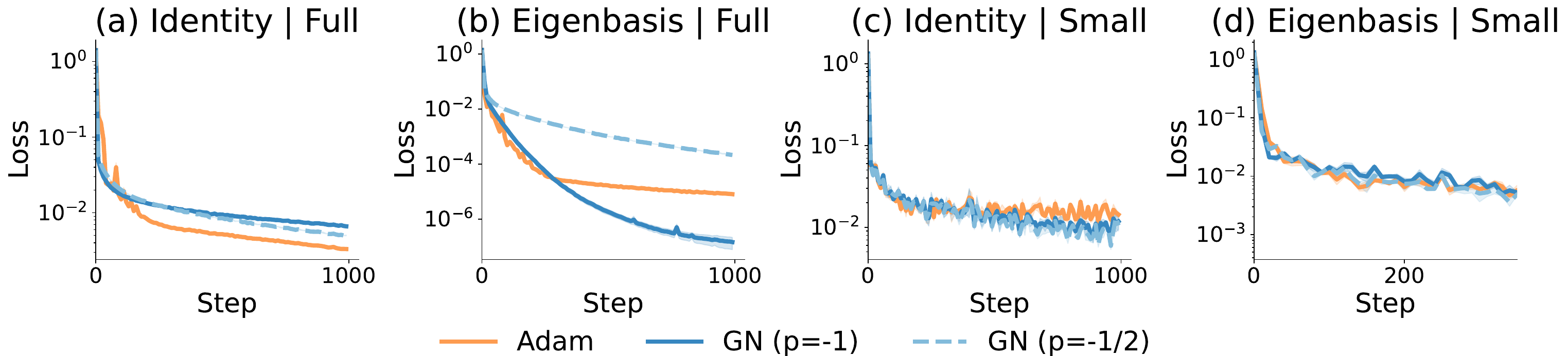}
  \caption{\textbf{Learning from a random teacher}, comparing Adam, $\GN^{-1}$ and $\GN^{-1/2}$ for the full $2 \times 2$ grid (\Cref{tab:basis-batch-grid}).
  }
  \label{fig:teacher-2x2-plots}
  \vspace{-1em}
\end{figure}

\begin{figure}[t]
  \centering
  \includegraphics[width=\textwidth]{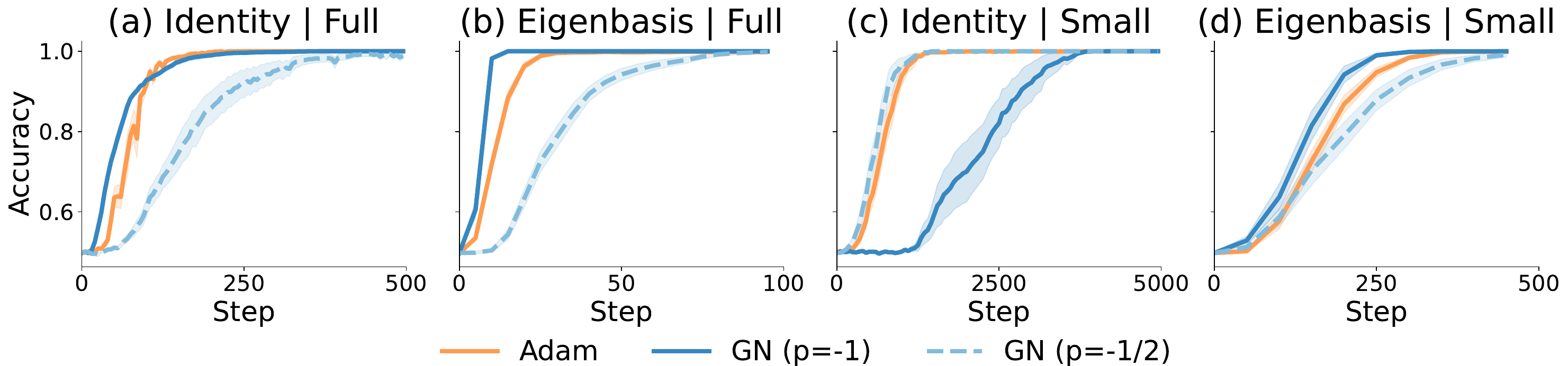}
  \caption{\textbf{Sparse parity}, comparing Adam, $\GN^{-1}$ and $\GN^{-1/2}$ for the full $2 \times 2$ grid (\Cref{tab:basis-batch-grid}).
  }
  \label{fig:parity-2x2-plots}
\end{figure}

\begin{figure}[t]
  \centering
  \includegraphics[width=\textwidth]{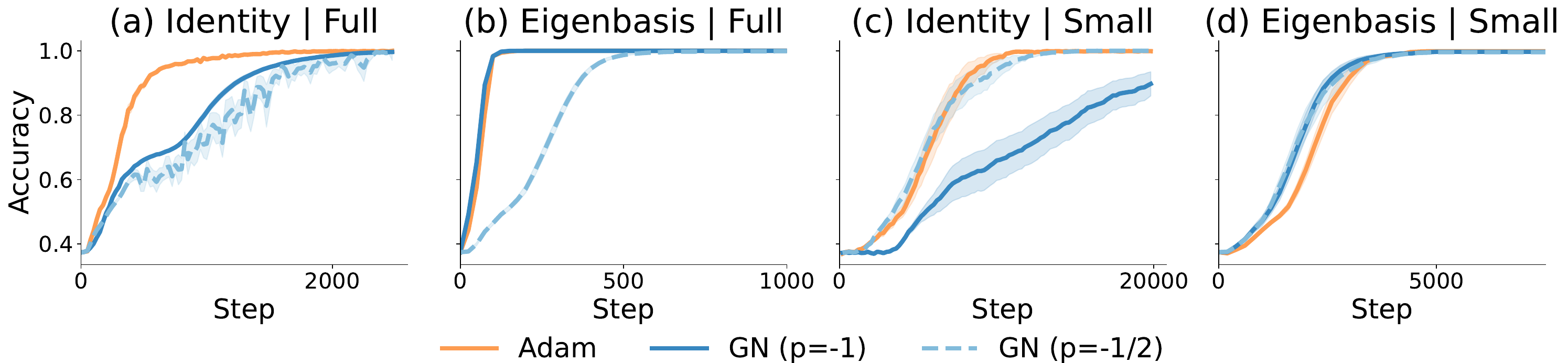}
  \caption{\textbf{Staircase}, comparing Adam, $\GN^{-1}$ and $\GN^{-1/2}$ for the full $2 \times 2$ grid (\Cref{tab:basis-batch-grid}).
  Staircase is a generalization of sparse parity (\Cref{fig:parity-2x2-plots}).
  }
\label{fig:staircase-2x2-plots}
\end{figure}

\begin{figure}[t]
    \centering
    \includegraphics[width=0.95\linewidth]{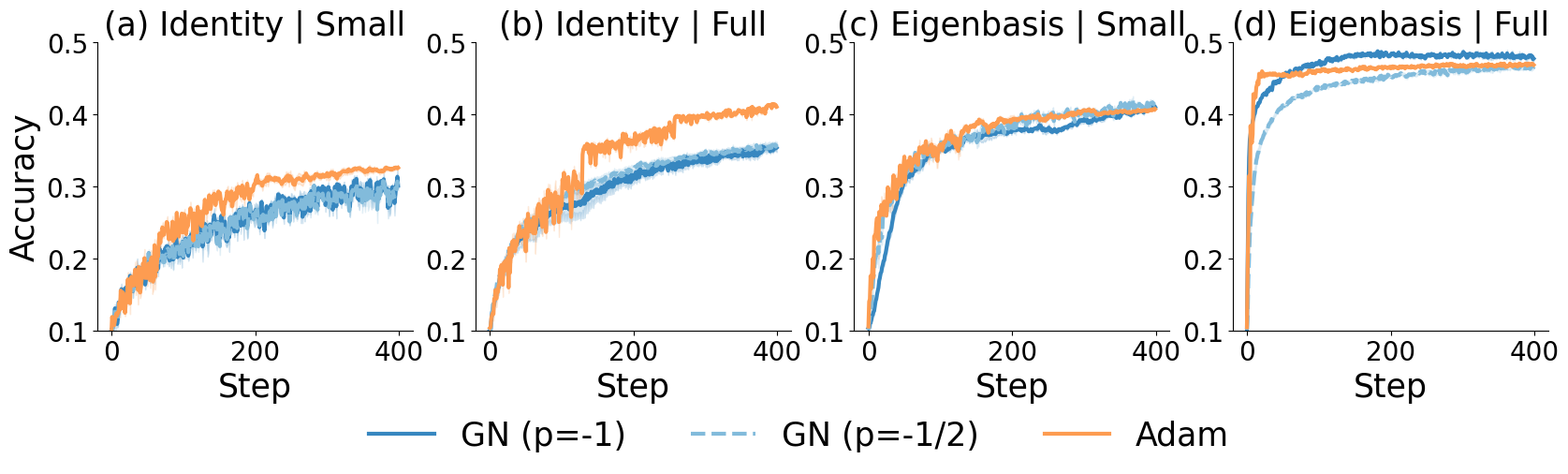}
    \caption{\textbf{CIFAR10}, comparing Adam, $\GN^{-1}$ and $\GN^{-1/2}$ for the full $2 \times 2$ grid (\Cref{tab:basis-batch-grid}).}
    \label{fig:cifar}
\end{figure}

\paragraph{A logistic-like task with Transformers}
\label{sec:expr_attn}
Finally, we experiment with Transformers, whose attention module resembles the structure of the reparametrized logistic regression in \S\ref{sec:logistic}.
We consider a selection-based regression task, where the input sequence contains a series of Gaussian vectors $\{x_t\}_{t\in[\seqlen]}$, followed by a one-hot vector that selects the position $t$ that the regression target is based on; details are provided in \S\ref{app:expr_details_logistic}).
As shown in \Cref{fig:transformer}, Adam outperforms GN with full batch updates even under the eigenbasis, consistent with our theoretical results.

\section{Discussion}
\label{sec:discussion}

This work compares Adam-style and GN-style diagonal preconditioners along two factors: the alignment to the ideal eigenbasis, and the gradient noise level as influenced by the batch size.
Our theoretical results on linear and logistic regression show that the comparison is sensitive to the change in either factor (\Cref{tab:basis-batch-grid}).
For linear regression,
we show that $\GN^{-1}$'s optimality under population gradients and the eigenbasis no longer holds either when using the incorrect identity basis (\S\ref{sec:full_batch_incorrect}),
or under the stochastic regime (\S\ref{sec:stochastic}),
where Adam-style and $\GN^{-1/2}$ behave similarly regardless of the basis choice, thereby revealing a connection between Adam's design and curvature-based preconditioning.
Moving beyond the linear regression, Adam-style scaling can outperform $\GN^{-1}$ even under the ideal setup of population gradients and the correct eigenbasis.
Our empirical results support the theoretical findings.
In particular, all MLP experiments align with linear regression results, and the Transformer experiments align with our logistic regression results.

It is important to understand whether phenomena and differences observed in small-scale, synthetic setups persist across scale.
We hypothesize that the equivalence between Adam-style and $\GN^{-1/2}$ in the stochastic regime extends to practical, large-scale training.
In particular, as training progresses, the gradient variance tends to dominate over the gradient mean, mirroring the stochastic regime in which variance drives the dynamics.
Validating this hypothesis in large-scale settings is an interesting direction for future work.
Finally, such equivalence combined with the benefits of Adam's \textit{auto-tuning} at large-batch regimes suggests a promising direction: developing algorithms that exhibit similar desirable auto-tuning behavior even when operating with small batches.

\section*{Acknowledgment}
We thank Alex Damian for the helpful feedback.
This work was enabled in part by a
gift from the Chan Zuckerberg Initiative Foundation to establish the Kempner Institute for the Study
of Natural and Artificial Intelligence.
DM is supported by a Simons Investigator Fellowship, NSF grant DMS-2134157, DARPA grant W911NF2010021,and DOE grant DE-SC0022199.
SK and DM acknowledge support from the Office of Naval Research under award N0001422-1-2377 and the National Science Foundation Grant under award \#IIS 2229881.

\bibliographystyle{iclr2026_conference}
\bibliography{references}

\newpage

\appendix

\section{Theoretical results and omitted proofs}

\subsection{Optimality of \texorpdfstring{$\GN^{-1}$}{GN(-1)} in the correct basis}
\label{app:gn_correct_basis}

For completeness, we provide proofs for the optimality of $\GN^{-1}$ for the quadratic loss under the correct eigenbasis.

\paragraph{Full batch}
For $\GN^{-1}$, the parameter estimation error evolves as
\begin{equation}
\begin{split}
    \param[1] -\param^* =& (\param[0] - \param^*) - \lr \cdot (\hGN)^{-1} \cdot \grad[0]
    \\
    =& \param[0] - \lr \cdot \E[xx^\top]^{-1} \cdot \E[xx^\top(\param - \param^*)]
    = (1-\lr) (\param - \param^*).
\end{split}
\end{equation}
Hence $\GN^{-1}$ can reach the optimum in 1 step with $\lr = 1$.

\paragraph{Stochastic regime}
Let's consider the stochastic regime where each update step is performed with a single sample.
With respect to some algorithm, define
\[
\paramCov[t] := \E[(\param[t] - \param^*) (\param[t] - \param^*)^\top].
\]
which is the expected second-moment matrix of the distance-to-opt.

Recall that $\GN^{-1}$ has updates $\param[t+1] = \param[t] - \lr \Cov_x^{-1} \grad[t]$, where $\grad = (\param - \param^*)^\top x \cdot x$.
We have that:
\begin{equation}
    \paramCov[t+1] = \paramCov[t] - 2 \lr \paramCov[t] + \lr^2 \Tr(\paramCov[t] \Cov_x) \Cov_x^{-1} + 2\lr^2 \paramCov[t].
\end{equation}
Multiplying by $\Cov_x$ and taking the trace leads to:
\begin{equation}
    \E[\loss[t+1]] = \big(1-2\lr+\lr^2(\dim+2)\big) \E[\loss[t]].
\end{equation}
Setting $\lr = \frac{1}{\dim+2}$ reduces the expected error by a $\frac{1}{2}$ factor every $O(\dim)$ steps.

\subsection{Equivalence of Adam and \texorpdfstring{$\GN^{-0.5}$}{GN(-0.5)} under stochastic regime}
\label{app:stochastic_regime}

We start with establishing the equivalence between the empirical and true Fisher (\Cref{lem:adam_gn_half_fisher}), which will then be used to prove the equivalence of Adam and $\GN^{-\frac{1}{2}}$'s updates.

\subsubsection{Proof of \texorpdfstring{\Cref{lem:adam_gn_half_fisher}:}{Lemma on the} equivalence of the empirical and true Fisher}

\begin{lemma*}[\Cref{lem:adam_gn_half_fisher}, restated]
    For linear regression with Gaussian inputs, the following holds:
    \[ \loss(\param) \cdot \Cov_x \preceq \frac{1}{2}\E[\grad(x) \grad(x)^\top] \preceq 3\loss(\param) \cdot \Cov_x.\]
\end{lemma*}
\begin{proof}
    For a given $\param$, let $\paramCov := (\param-\param^*)(\param-\param^*)^\top$.
    Then w.r.t. $\param$, we have
    \begin{equation}
        \E[\grad(x) \grad(x)^\top]
        = \E[x^\top \paramCov x \cdot xx^\top]
        = 2\Cov_x \paramCov \Cov_x + \Tr(\Cov_x\paramCov) \Cov_x
        \preceq 3 \Tr(\Cov_x\paramCov) \Cov_x,
    \end{equation}
    where the last equality follows from Wick's theorem. 
    The lemma follows by noting that $\ell = \frac{1}{2} \E[x^\top \paramCov x] = \frac{1}{2}\Tr(\Cov_x\paramCov)$.
\end{proof}

\subsubsection{Proof of \texorpdfstring{\Cref{lem:adam_gn_half_update}:}{Corollary on the} equivalence of Adam and \texorpdfstring{$\GN^{-\frac{1}{2}}$}{GN(-0.5)}}
\label{app:proof_adam_gn_half_update}

\begin{corollary*}[\Cref{lem:adam_gn_half_update}, restated]
    For single-sample updates, the update of Adam and $\GN^{-\frac{1}{2}}$ differ by a constant.
\[ \frac{1}{\sqrt{3\loss}} \cdot \precondD[\GN, -\frac{1}{2}] \preceq \frac{1}{2}\precondD[A] \preceq \frac{1}{\sqrt{\loss}} \cdot \precondD[\GN, -\frac{1}{2}].\]
        
\end{corollary*}
\begin{proof}

Adam's preconditioner is based on
\begin{equation}
    \precond[A] := \E_{(x,y) \sim \distr}[\grad(x) \grad(x)^\top].
\end{equation}

Given a basis $\ebasis$, the diagonal preconditioner given by Adam has entries
\begin{equation}
    (\precondD[A]_{ii})^{-1} = \sqrt{\evec_i^\top \precond[A] \evec_i}.
\end{equation}
The relation between $\precondD[A]$ and $\precondD[\GN, -0.5]$ follows from \Cref{lem:adam_gn_half_fisher} and the fact that $(\precondD[\GN, -0.5]_{ii})^{-1} = \sqrt{\evec_i^\top \Cov_x \evec_i}$.

\end{proof}

\paragraph{Part 2: when $\hGN$ is based on a single sample}
On single-sample batches, the gradient is $\grad(\param) = \loss_f' \cdot \nabla_{\param} f$.
Then, the diagonal preconditioner for Adam (with $\beta_1 = \beta_2 = 0$) has
Given a basis $\ebasis$, let $\tilde \grad := \ebasis^\top \grad$ denote the gradient rotated into the basis.
\begin{equation*}
    \precondD[A]_{ii} = (|\tilde \grad_i|)^{-1} = \left((\loss_f')^2 \cdot (\ebasis^\top \nabla_{\param} f_i)^2\right)^{-0.5}.
\end{equation*}

The diagonal preconditioner for Gauss-Newton has 
\begin{equation*}
    \precondD[GN]_{ii} = (\hGN_{ii})^{-0.5}
    = (\loss_f'' \cdot (\ebasis^\top \nabla_{\param} f_i)^2)^{-0.5}
    = ((\loss_f')^2/\loss_f'')^{0.5} \cdot \precondD[A]_{ii}.
    \end{equation*}

\subsubsection{Proof of \texorpdfstring{\Cref{lem:precond_loss_rate}:}{Lemma on} loss convergence of general preconditioners}

\begin{lemma*}[\Cref{lem:precond_loss_rate}, restated]
    For a general preconditioner $\precond$, for linear regression with stochastic Gaussian inputs, the following holds:
    \[ \E[\loss[t]] \leq O\left[\left(1 - \frac{\lambda_{\min}(\mA(\precond))}{3\Tr(\mA(\precond))}\right)^t \loss[0]\right]. \]
\end{lemma*}

\begin{proof}
    Let's define
    \begin{equation}
        \paramCov[t] = \E[(\param[t] - \param^\star) (\param[t] - \param^\star)^\top].
    \end{equation}
    For any preconditioner $\precond$, given the update $\param[t+1]-\param^* = (\Id - \lr \precond xx^\top)(\param[t]-\param^*)$,
    we have:
    \begin{equation}
    \begin{split}
        \paramCov[t+1] &= \paramCov[t] - \lr \precond\Cov_x \param[t] - \lr \param[t] \Cov_x\precond +\lr^2 \precond\bigg( 
        2\Cov_x\paramCov[t] \Cov_x+ \Tr(\Cov_x\paramCov[t]) \Cov_x\bigg)\precond\\
        &= (\Id- \lr \precond\Cov_x) \paramCov[t] (\Id- \lr \precond\Cov_x) ^\top +\lr^2 P\bigg( 
        \Cov_x\paramCov[t] \Cov_x+ \Tr(\Cov_x\paramCov[t]) \Cov_x\bigg)P. \\
    \end{split}
    \end{equation}
    Observe that $\E[\loss_t] = \E[\Tr(\Cov_x \paramCov[t])]$.
    This motivates us to make the following definition:
    \begin{equation}
        \tparamCov[t] = \Cov_x^{1/2} \paramCov[t] \Cov_x^{1/2},
    \end{equation}
    and so $\E[\loss_t] = \E[\Tr(\tparamCov[t])].$
    
    Define $\mA(\precond) := \Cov_x^{1/2} \precond\Cov_x^{1/2}$.
    The corresponding update rule is then:
    \begin{equation}
    \begin{split}
        \tparamCov[t+1]
        &= (\Id- \lr \mA(\precond)) \tparamCov[t] (\Id- \lr \mA(\precond)) ^\top
        + \lr^2
        \mA(\precond)\bigg( 
        \tparamCov[t] + \Tr(\tparamCov[t]) \Id \bigg) \mA(\precond)
        \\
        &\preceq (\Id- \lr \mA(\precond)) \tparamCov[t] (\Id - \lr \mA(\precond))^\top
         + 2\lr^2 \Tr(\tparamCov[t]) (\mA(\precond))^2.
    \end{split}
    \end{equation}

    For a given $\precond$, achieving the best loss contraction rate reduces to finding the optimal $\lr$.
    Rotating the left and the right hand side into the eigenbasis of $\mA(\precond)$ and noting that the Trace of a matrix is independent of rotation, we can consider diagonal $\mA(\precond)$ (without loss of generality)
    with the diagonal entries correspond to the eigenvalues.
    Define:
    \[
    v_t = \diag(\tparamCov[t]).
    \]
    We have:
    \begin{equation}
        v_{t+1} =  ( (\Id-\lr \mA(\precond))^2 +\lr^2 \mA(\precond)^2 + \lr^2 \textrm{diag}(\mA(\precond)^2)\vec{1}^\top) v_t.
    \end{equation}
    
    Taking dot product with $\diag(\mA(\precond)^{-1})$ on both sides, we get
    \begin{equation}
    \begin{split}
        \diag(\mA(\precond)^{-1})^{\top}v_{t+1} & = \diag(\mA(\precond)^{-1})^{\top}( (\Id-\lr \mA(\precond))^2 +\lr^2 \mA(\precond)^2 + \lr^2 \diag(\mA(\precond)^2)\vec{1}^\top) v_t \\
        &= \diag(\mA(\precond)^{-1})^{\top} v_t - 2\lr 1^\top v_t + 2\lr^2 \diag(\mA(\precond))^\top v_t + \lr^2 \Tr(\mA(\precond)) 1^\top v_t.
    \end{split}
    \end{equation}

    Since $\tparamCov[t]$ is PSD, $v_t$ has non-negative entries.
    Hence we have
    $$1^\top v_t = \diag(\mA(\precond)^{-1})^\top \mA(\precond) v_t \geq \eval_{\min}(\mA(\precond)) \diag(\mA(\precond)^{-1})^\top v_t,$$
    and
    \begin{equation}    
    \begin{split}
        &\diag(\mA(\precond)^{-1})^{\top}v_{t+1}
        \\
        \leq& \diag(\mA(\precond)^{-1})^{\top} v_t - 2\lr 1^\top v_t + \eval_{\max}(\mA(\precond)) 2\lr^2 1^\top v_t + \lr^2 \Tr(\mA(\precond) 1^\top v_t
        \\
        =& \diag(\mA(\precond)^{-1})^{\top} v_t - \lr \cdot \big(2 - (2\eval_{\max}(\mA(\precond)) + \Tr(\mA(\precond)) \lr \big) 1^\top v_t 
        \\ 
        \leq& \bigg(1 - \eval_{\min}(\mA(\precond)) \cdot \lr \big(2 - (2\eval_{\max}(\mA(\precond)) + \Tr(\mA(\precond)) \lr \big) \bigg) \cdot \diag(\mA(\precond)^{-1})^{\top} v_t.
    \end{split}
    \end{equation}
    The max contraction rate is achieved by setting $\lr = \frac{1}{2\eval_{\max}(\mA(\precond)) + \Tr(\mA(\precond))}$, which gives
    \begin{equation}
    \begin{split}
        \diag(\mA(\precond)^{-1})^{\top}v_{t+1}
        \leq& \bigg(1 - \frac{\eval_{\min}(\mA(\precond))}{2\eval_{\max}(\mA(\precond)) + \Tr(\mA(\precond))}\bigg) \diag(\mA(\precond)^{-1})^{\top}v_{t}
        \\
        \leq& \bigg(1 - \frac{\eval_{\min}(\mA(\precond))}{3\Tr(\mA(\precond))}\bigg) \cdot \diag(\mA(\precond)^{-1})^{\top}v_{t}.
    \end{split}
    \end{equation}

\end{proof}

\subsection{Comparing GN powers}
\label{app:gn_power}

\begin{figure*}[t]
  \centering
  \includegraphics[width=0.6\textwidth]{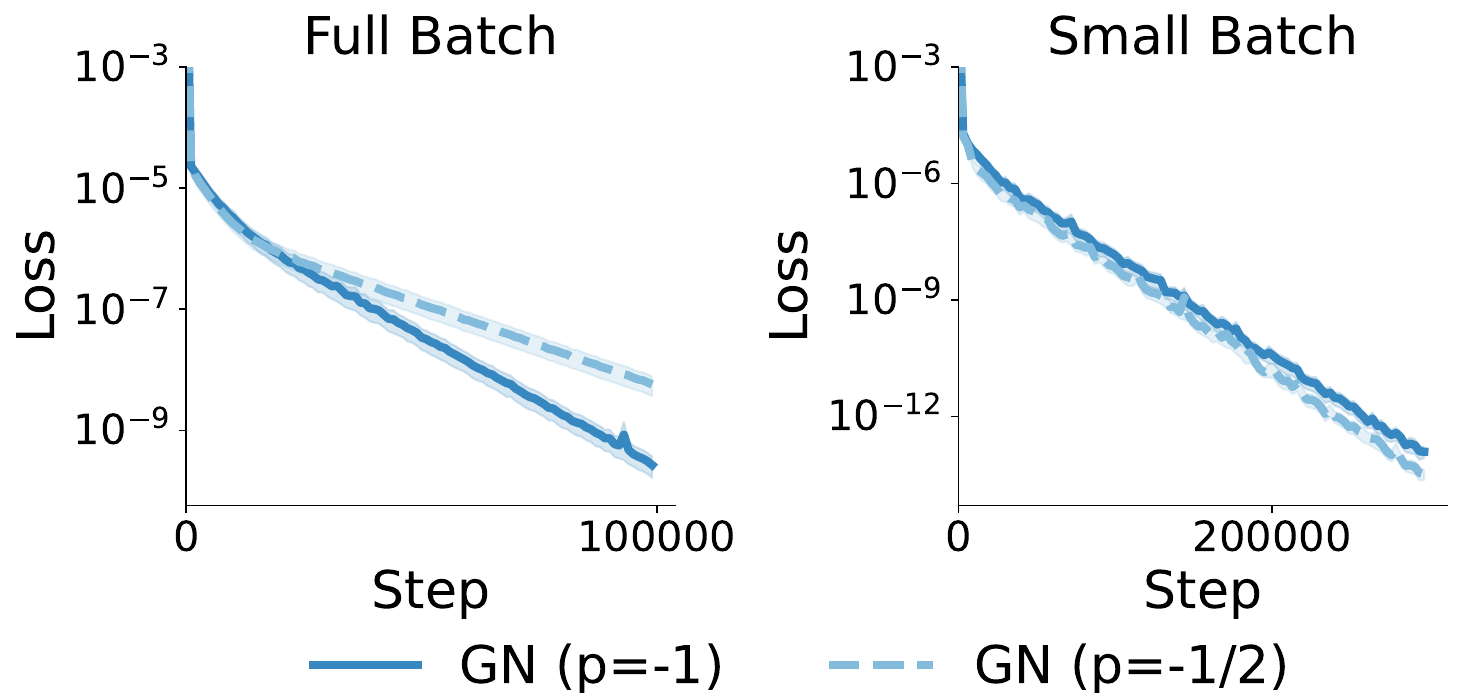}
  \caption{Comparing GN power $\power \in \{-\frac{1}{2}, -1\}$. Contrary to \Cref{fig:GN-powers}, when GN$^{-1}$ has a more favorable condition number, it converges faster or close to GN$^{-\frac{1}{2}}$ on both small and large batches.}
  \label{fig:powers2}
\end{figure*}

This section discusses the comparison between $\GN^{-1}$ and $\GN^{-\frac{1}{2}}$.
We will show that under the identity basis, 
$\GN^{-1/2}$ can outperform $\GN^{-1}$ even with full batches.

Let $\mA(\precond) := \precond^{1/2} \Cov_x \precond^{1/2}$ as defined in \Cref{sec:stochastic}.
One can show that the convergence rate of preconditioned gradient descent \citep{Boyd_Vandenberghe_2004} depends on the condition number of the preconditioned Hessian given by
\begin{equation}
    \kappa(\mA(\precond)) := \frac{\lambda_{\max}(\mA(\precond))}{\lambda_{\min}(\mA(\precond))}.
\end{equation}

By \Cref{lem:precond_loss_rate}, to compare these two powers, it suffices to compare the condition number of for specific preconditioners $\Cov_x^{-1}$ and $\Cov_x^{-1/2}$.
We claim that there exist problems for which, even in the full batch case, in identity basis, $\power = 0.5$ leads to faster convergence than $\power = 1$:
\begin{claim}
    There exists $\Cov_x$ such that $\kappa(\mA(\diag(\Cov_x)^{-1/2})) < \kappa(\mA(\diag(\Cov_x)^{-1}))$.
\end{claim}

Denote $r(\Cov_x) := \frac{\cond(\Cov_x^{1/2} \diag(\Cov_x^{-1}) \Cov_x^{1/2})}{\cond(\Cov_x^{1/2} \diag(\Cov_x^{-1/2}) \Cov_x^{1/2})}$.
We empirically show that there exists $\Cov_x$ such that $r(\Cov_x) > 1$.
We obtain such $\Cov_x$ by fixing the diagonal matrix of eigenvalues $\Eval$ and randomly sampling orthonormal matrices $\ebasis$, and setting $\Cov_x = \ebasis \Eval \ebasis^\top$.

In particular, we construct two covariance matrices $\Cov_{\text{half}}, \Cov_{\text{one}} \in \R{5 \times 5}$, such that $r(\Cov_{\text{half}}) > 1$ (i.e. $\GN^{-1/2}$ is more favorable), and $r(\Cov_{\text{one}}) < 1$ (i.e. $\GN^{-1}$ is more favorable).
As shown in \Cref{fig:GN-powers} and \Cref{fig:powers2},
$\GN^{-\frac{1}{2}}$ indeed converges faster on data from $\Cov_{\text{half}}$, whereas $\GN^{-1}$ converges faster with $\Cov_{\text{one}}$, consistent with the theory.

Characterizing covariance matrices for which $r(\Cov_x) > 1$ is left as future work.

\section{Proof for the logistic lower bound for GN \texorpdfstring{(\Cref{theorem:lower_logistic})}{}}
\label{sec:logistic_proof}

We optimize a weight-tied two-layer linear network $\pf: \R^\dim \rightarrow \R^\dim$,
where 
\begin{equation}
    \pf_i(\param) = \Pr_{\param}(y=1\mid x=e_i)
        = \sigma\!\Bigl(\sum_{j=1}^{\dim} \param_j^{2}x_j\Bigr)
        = \sigma(\param_i^{2}),
    \qquad
    \sigma(z)=\frac1{1+\exp(-z)}.
\end{equation}
For logistic regression, 
the loss is $\loss(\param) = -\E_{x,y}[\sum_{i\in[\dim]} \mathbbm{I}(x=e_i) \big(y\log \pf_i(\theta) + (1-y)\log(1-\pf_i(\theta))\big)]$.
The updates separate by dimension;
the population gradient $\grad(\param)$ and the diagonal Gauss-Newton matrix $\hGN(\param)$ are given by
\begin{equation} 
\label{eq:gn_update_logistic}
    [\grad(\param)]_i = 2\nu_i\,\param_i\bigl(\sigma(\param_i^{2})-\py_i\bigr),
    \quad
    [\hGN(\param)]_{ii} = 4\nu_i\,\param_i^{2}\, \sigma(\param_i^{2})\bigl(1-\sigma(\param_i^{2})\bigr),
\end{equation}
where $\py_i := \Pr(y=1|x=e_i)$.

For a step size sequence $\{\lr[t]\}$ and ridge regularization $\reg \ge 0$, the Gauss-Newton iteration is
\[
    \param[t+1]= M_{\lr[t],\reg}(\param[t]),
    \qquad
    \text{where } M_{\lr,\reg}(\param) = \param-\lr\,(\hGN(\param)+\reg I)^{-1} \grad(\param).
\]
Coordinate-wise, this update reads
\begin{equation}
\label{eq:GN-diag}
    [\,M_{\lr,\reg}(\param)\,]_{i}
    \;=\;
    \param_i
    -\\
     \frac{2\param_i\nu_i(\sigma(\param_i^{2})-\py_i)}
          {4(\param_i)^2\nu_i\,\sigma(\param_i^{2})(1-\sigma(\param_i^{2}))
           +\reg}.
\end{equation}

The proof strategy for \Cref{theorem:lower_logistic} hinges on the following key technical lemma that establishes a learning rate threshold for a single coordinate, where the threshold is a function of both initialization $\param[0]$ and the regularization parameter $\reg$.
Above this threshold, the Gauss-Newton update diverges.
By requiring that all coordinates avoid this divergence to ensure global convergence, we use the lemma to derive a strict upper bound on the algorithm's final learning rate, $\lr^\infty$.
Substituting this necessary restriction into the definition of the local contraction factor will directly yield the lower bound in \Cref{theorem:lower_logistic}, showing that slow convergence is an unavoidable consequence of global stability from the chosen initialization. 

\begin{lemma}\label{lemma:divergence}
For constants $\lr, \reg > 0$, define the one-dimensional update map $M_{\lr,\reg}: \R \to \R$ corresponding to a regularized Gauss-Newton step:
\[
M_{\lr,\reg}(\param)
=
\param - \lr \frac{2\param(\sigma(\param^2)-\py)}{4\param^2\sigma(\param^2)(1-\sigma(\param^2)) + \reg}.
\]
Consider the update rule $\param[t+1] = M_{\lr[t],\reg}(\param[t])$.

There exists a universal constant $c$ such that for any target probability $\py \in [0.6, 0.8]$ and any initial weight $\param[0] > 0$ satisfying $\sigma((\param[0])^2) \le 0.55$, if the non-increasing learning rate sequence satisfies
\[
\lr[t] \ge c \sqrt{\log \frac{1}{\param[0]}} \left( \param[0] + \frac{\reg}{\param[0]} \right) \quad \text{for all } t,
\]
then $|\param[t]|$ diverges geometrically.
\end{lemma}

\begin{proof}
The proof proceeds in three parts.
First, we show that the first step, $\param[1]$, becomes large.
Second, we establish a key property satisfied by $\param[1]$. 
Finally, we use this to show that all subsequent iterates grow geometrically.

\emph{Part 1: The first step makes $\param[1]$ large.}
The condition $\sigma((\param[0])^2) \le 0.55$ implies $\param[0] \le
0.5$. Since $\py \ge 0.6$, the term $\sigma((\param[0])^2) - \py$ is
negative, ensuring $\param[1] > \param[0]$. We can lower bound $\param[1]$
as follows:  
\begin{align*}
    \param[1] &= \param[0] - \lr[0] \frac{2\param[0](\sigma((\param[0])^2)-\py)}{4(\param[0])^2\sigma((\param[0])^2)(1-\sigma((\param[0])^2)) + \reg} \\
    &\ge \param[0] + \lr[0] \frac{2\param[0](0.6 - 0.55)}{4(\param[0])^2 \cdot 0.5^2 + \reg} \\
    &\ge \lr[0] \frac{0.1 \param[0]}{(\param[0])^2 + \reg}\\
    &= \lr[0] \frac{0.1}{\param[0] + \reg/\param[0]}.
\end{align*}
Substituting our lower bound for $\lr[0]$ from the lemma statement yields:
\[
    \param[1] \ge \left( c \sqrt{\log \frac{1}{\param[0]}} \left( \param[0] + \frac{\reg}{\param[0]} \right) \right) \frac{0.1}{\param[0] + \reg/\param[0]} = 0.1 c \sqrt{\log \frac{1}{\param[0]}}.
\]
As $\param[0] \to 0$, this lower bound grows, so we can choose the universal constant $c$ large enough to make $\param[1]$ arbitrarily large. 

\emph{Part 2: Establishing a key property of $\param[1]$.}
We now show we can choose $c$ large enough to ensure two conditions
hold simultaneously for $\param[1]$:  
\begin{enumerate}
    \item[(i)] $\sigma((\param[1])^2) \ge 0.9$.
    \item[(ii)] $4(\param[1])^2 (1-\sigma( (\param[1])^2)) \le \param[0]$.
\end{enumerate}
Condition (i) is met by choosing $c$ sufficiently large.
For condition (ii), we use the facts that $1-\sigma(z) \le e^{-z}$ and that
$z(1-\sigma(z))$ is a decreasing function for $z \ge 2$.~\footnote{Indeed, $\frac{\mathrm d}{\mathrm dz}[\,z(1-\sigma(z))\,] = (1-\sigma(z)) - z\sigma(z)(1-\sigma(z))< 0$ once $z \ge 2$.}
Let $\param[1]_{\text{low}} := 0.1c \sqrt{\log(1/\param[0])}$.
This implies: 
\begin{align*}
  4(\param[1])^2 (1-\sigma((\param[1])^2))
& \le 4(\param[1]_{\text{low}})^2 (1-\sigma((\param[1]_{\text{low}})^2))
  \le 4 (\param[1]_{\text{low}})^2 e^{-(\param[1]_{\text{low}})^2} \\
    &= 4 (0.01c^2)\log(1/\param[0]) \cdot \exp\left(-(0.01c^2)\log(1/\param[0])\right) \\
    &= (0.04c^2)\log(1/\param[0]) \cdot (\param[0])^{0.01c^2}\\
    &= \bigg((0.04c^2)\log(1/\param[0]) (\param[0])^{0.01c^2-1}\bigg) \cdot \param[0].
\end{align*}
For a sufficiently large constant $c$ (e.g., $0.01c^2 > 2$), the term in the large parenthesis is less than 1, because for a fixed $\param[0] \in (0, 0.5]$, the polynomial term $(\param[0])^{0.01c^2-1}$ decays much
faster than the logarithmic term $\log(1/\param[0])$ grows.
This establishes condition (ii).  

\emph{Part 3: Proving geometric divergence for $t \ge 1$.}
Consider any $\lr$ satisfying the learning rate lower bound, i.e. suppose:
\[
\lr \ge c \sqrt{\log \frac{1}{\param[0]}} \left( \param[0] + \frac{\reg}{\param[0]} \right) .
\]
We show that for any $\param$ where $\param^2 \ge (\param[1])^2$, it follows that $|M_{\lr,\reg}(\param)| \ge \sqrt{2}|\param|$.  
First, rewrite the update as 
\[
  M_{\lr,\reg}(\param) = \param \left(1 - \frac{2\lr(\sigma(\param^2)-\py)}{4\param^2\sigma(\param^2)(1-\sigma(\param^2)) + \reg}\right).
\]
Let $K(\param) := \frac{2\lr(\sigma(\param^2)-\py)}{4\param^2\sigma(\param^2)(1-\sigma(\param^2)) + \reg}$;
note that $K(\param) > 0$ provided that $\param^2 \geq (\param[1])^2$.
We seek to show $|1 - K(\param)| \ge \sqrt{2}$, which is true if $K(\param) \ge 1+\sqrt{2}$.

We lower bound $K(\param)$ for any $\param$ where $\param^2 \ge (\param[1])^2$.
The numerator can be lower-bounded using condition (i): 
\begin{equation}
    2\lr(\sigma(\param^2)-\py) \ge 2\lr(\sigma((\param[1])^2)-\py) \ge 2\lr(0.9-0.8) = 0.2\lr.
\end{equation}
For the denominator, we use the fact that $z(1-\sigma(z))$ is decreasing (for $z \ge
2$), condition (ii), and that $\param[0] \in [0, 0.5]$:
\[
  4\param^2\sigma(\param^2)(1-\sigma(\param^2)) + \reg \le 4\param^2(1-\sigma(\param^2)) 
  + \reg \le 4(\param[1])^2\bigl(1-\sigma((\param[1])^2)\bigr) + \reg \le \param[0] + \reg \le \param[0] + \frac{\reg}{\param[0]}.
\]
Combining these bounds gives:
\[
K(\param) \ge \frac{0.2\lr}{\param[0] + \reg/\param[0]}\ge 0.2 c \sqrt{\log \frac{1}{\param[0]}},
\]
where the last step follows by substituting the lower bound for $\lr$.
Since $\param[0] \le 0.5$, we have $\log(1/\param[0]) \ge \log(2)$.
We can choose the universal constant $c$ large enough such that $0.2 c\sqrt{\log 2} \ge 1+\sqrt{2}$. 

Thus, for any $t \ge 1$, we have $|\param[t+1]| = |M_{\lr[t],\reg}(\param[t])| \ge \sqrt{2}|\param[t]|$,
which shows that $|\param[t]|$ diverges geometrically.
\end{proof}

We are now ready to complete the proof of Theorem~\ref{theorem:lower_logistic}.
\begin{proof}
(Theorem~\ref{theorem:lower_logistic}) The proof proceeds by using Lemma~\ref{lemma:divergence} to find an
upper bound on the final learning rate $\lr[\infty]$, and then
substituting this bound into the definition of the local contraction
factor $\spectralRad$. 

At the optimum $\param^*$, the diagonal entries of the Fisher matrix $\hGN_*=\hGN(\param^*)$ are
\begin{equation}
\label{eq:logreg_eval_opt}
    \eval_i(\hGN_*) = 4 (\param_i^*)^{2} \nu_i \sigma((\param_i^*)^2)\bigl(1-\sigma((\param_i^*)^2)\bigr).
\end{equation}

Recall the target probabilities $\py_i = \sigma((\param_i^*)^2)$ are assumed to lie in $[0.6, 0.8]$.
This implies that $(\param_i^*)^2$'s are bounded by a universal constant.
Thus, the term $4 (\param_i^*)^2 \sigma((\param_i^*)^2)\bigl(1-\sigma((\param_i^*)^2)\bigr)$ is also bounded by universal constants,
and we conclude that the Fisher eigenvalues are proportional to the sampling probabilities $\{\nu_i\}$.
In particular, 
\[
  \eval_{\min}(\hGN_*) \ge \nu_{\min}/c_1,
\]
where $c_1$ is a universal constant.

The Gauss-Newton update for each coordinate $\param_i$ can be analyzed independently.
For the sequence $\param[t]$ to converge $\param^*$, the iterates for each coordinate $\param_i[t]$ must also converge to $\param_i^*$.
From~\Cref{eq:GN-diag} and taking $\reg/\nu_i$ as the regularization parameter, the update rule for each coordinate can be written as: 
\[
    [\,M_{\lr,\reg}(\param)\,]_i
    \;=\;
    \param_i
    -\lr\,
     \frac{2\param_i(\sigma(\param_i^{2})-\py_i)}
          {4(\param_i)^2\,\sigma(\param_i^{2})(1-\sigma(\param_i^{2}))
           +\reg/\nu_i}.
\]
We can apply Lemma~\ref{lemma:divergence} coordinate wise.
Since $\{\lr[t]\}$ is non-increasing, Lemma~\ref{lemma:divergence} implies the limiting stepsize $\lr[\infty]$ must satisfy the following for each coordinate $i \in [\dim]$: 
\[
\lr[\infty] \le c\sqrt{\log \frac{1}{\param[0]_i}} \left( \param[0]_i + \frac{\reg / \nu_i}{\param[0]_i} \right).
\]

In our setting, the initial weights are $\param[0]_i = 1/\sqrt{\dim}$.
To get a single upper bound on $\lr[\infty]$, we take the tightest possible constraint derived from above,
which is when $\nu_i$ is at its maximum, $\nu_{\max}$.
Thus, for convergence to be possible, $\lr[\infty]$ must be bounded by: 
\[
  \lr[\infty] \le
  c\sqrt{\log \dim}\left(\frac{1}{\sqrt{\dim}}+\frac{\sqrt{\dim}\reg}{\nu_{\max}}\right).
\]

The local contraction factor is the spectral radius of
$I -\lr[\infty] (\hGN_*+ \diag(\{\reg/\nu_i\}))^{-1}\hGN_*$,
whose eigenvalues are $1 -\lr[\infty] \frac{\eval_i(\hGN_*)}{\eval_i(\hGN_*)+\reg/\nu_i} \geq 1 - 1 -\lr[\infty] \frac{\eval_i(\hGN_*)}{\eval_i(\hGN_*)+\reg}$ (recall \Cref{eq:logreg_eval_opt}).
We can lower bound the spectral radius by considering the smallest eigenvalue of $\hGN_*$, $\lambda_{\min}$: 
\[
 \spectralRad(\eta^{\infty},\reg) \ge 1 - \eta^{\infty} \frac{\lambda_{\min}(\hGN_*)}{\lambda_{\min}(\hGN_*) + \reg}.
\]
Substituting the upper bound on $\lr[\infty]$ from above:
\begin{align*}
  \spectralRad(\lr[\infty],\reg)
   &\ge 1 - \left( c\sqrt{\log \dim}\left(\frac{1}{\sqrt{\dim}}+\frac{\sqrt{\dim}\reg}{\nu_{\max}}\right) \right) \frac{\eval_{\min}}{\eval_{\min} + \reg}\\
   &= 1 - c\sqrt{\log \dim} \left( \frac{1}{\sqrt{\dim}} \cdot
      \frac{\eval_{\min}}{\eval_{\min}+\reg}
      + \frac{\sqrt{\dim}\reg}{\nu_{\max}} \cdot \frac{\eval_{\min}}{\eval_{\min}+\reg} \right).
\end{align*}
We can bound the terms in the parenthesis using
$\frac{\eval_{\min}}{\eval_{\min}+\reg} \le 1$ and
$\frac{\reg}{\eval_{\min}+\reg} \le 1$: 
\begin{align*}
 \spectralRad(\lr[\infty],\reg) &\ge 1 - c\sqrt{\log \dim} \left( \frac{1}{\sqrt{\dim}} + \frac{\sqrt{\dim}\eval_{\min}}{\nu_{\max}} \right).
\end{align*}
Using the our lower bound $\eval_{\min} \ge \nu_{\min}/c_1$ from
above, and absorbing constants into $c'$:  
\[ 
\spectralRad(\lr[\infty],\reg) \ge 1 - c'\sqrt{\log \dim} \left( \frac{1}{\sqrt{\dim}} + \sqrt{\dim}\frac{\nu_{\min}}{\nu_{\max}} \right),
\]
which implies the claimed result.
\end{proof}

\section{Experiments}

\subsection{Additional experiment information}
\label{app:expr_details}

\paragraph{Hyperparameters, hardware, and runtime}
The learning rate ($\lr$) search is first performed at factors of 3 (e.g. $0.01, 0.003, 0.001$) and then at factors of 2 or finer around the optimal value.
The regularization ($\reg$) search is at factors at 10 (e.g. $10^{-3}, 10^{-4}$).
For Adam, we additionally sweep over $\beta_2 \in \{0, 0.9, 0.95, 0.99\}$.
When comparing batch sizes, we vary the batch size used for computing gradient, and always use a large batch size (4096) for Gauss-Newton matrix to ensure an accurate basis estimation.
Experiments were run on NVIDIA A100 GPUs.
Simulation runs in \S\ref{sec:autotune} each completes within 1min.
Simulation runs in \S\ref{sec:GN_power} take 9min for every 100k steps.
The parity and staircase runs take around 10min for every 1k steps.
For CIFAR experiments, runs under the identity basis take less than 5min each, and runs under the Kronecker approximation of the eigenbasis take around 80min each.

\paragraph{Kronecker factorization}
\label{app:kron}
Inspired by prior work~\citep{martens15KFAC,gupta2018shampoo,vyas2024soap},
our experiments use Kronecker factorization as a computationally efficient approximation to the full eigenbasis.
Given a matrix-valued parameter $W \in \R^{\dimw \times \dimh}$,
let $\grad \in \R^{\dimw\dimh}$ denote the flattened gradient,
and $\Grad \in \R^{\dimw \times \dimh}$ denote the unflattened gradient.
The $\dimw\dimh \times \dimw\dimh$ Gauss-Newton matrix can be approximated by a Kronecker factorization as
\begin{equation}
    \hGN := \E[\grad\grad^\top]
    \approx \E[\Grad\Grad^\top] \otimes \E[\Grad^\top\Grad],
\end{equation}
where $\otimes$ denote the Kronecker product.
The eigenvalues and eigenvectors of the Kronecker product are the products and Kronecker products of the factors;
hence the eigenbasis of $\hGN \in \R^{\dimw\dimh \otimes \dimw\dimh}$ can be approximated by computing the eigenbasis of the smaller $\E[\Grad\Grad^\top] \in \R^{\dimw \times \dimw}$, $\E[\Grad^\top\Grad] \in \R^{\dimh \times \dimh}$.

\textit{Is Kronecker approximation a good proxy for the full eigenbasis?}
\cite{benzing2022gradient} showed that Kronecker-factored approximation such as KFAC~\citep{martens15KFAC} can lead to better performance than using the full eigenbasis in some cases.
They attributed the gain to heuristic damping, which effectively controls the step sizes and was beneficial in their experiments.
Our experiments do not use such heuristic damping, and we find the Kronecker approximation to behave similarly to the full eigenbasis,
while being much more compute-efficient.

\subsection{Details for MLP experiments with squared loss}
\label{app:expr_details_mse}

This section provides details for the MLP experiments in \S\ref{sec:non_convex_mse}.

We learn all tasks with single-hidden-layer MLPs given by $\hat y = f(x; \param) = a \cdot \sigma(w^\top x + b)$,
where $w \in \mathbb{R}^{\dim\times \hiddendim}, a, b \in \R^\dim$, with $\dim$ and $\hiddendim$ being the input and hidden dimensions respectively.
The non-linearity $\sigma$ defaults to ReLU unless specified otherwise.

Below we provide detailed descriptions of the tasks:
\begin{itemize}
    \item \textbf{Learning from a random teacher network.}
    We first construct a teacher-student setting
    to evaluate our hypotheses on a non-convex example.
    Consider input vectors,
    $\mathbf{x}_i \sim \mathcal{N}(0, \Cov_x) \in \mathbb{R}^{\dim}$
    where $\Cov_x \in \mathbb{R}^{\dim \times \dim}$ is a random covariance matrix.
    We initialize a random teacher model,
    $f_\teacher: \R^{\dim} \rightarrow \R$,
    as a single hidden layer MLP.
    For each input vector, we sample output labels
    from the teacher: $y_i = f(x_i; \param_{\teacher}) = a_\teacher\cdot \sigma(w_\teacher^{\top}x_i + b_\teacher)$,
    where, $w_\teacher \in \mathbb{R}^{\hiddendim\times \dim}; a_\teacher, b_\teacher\in \mathbb{R^\hiddendim}$,
    and $n_\teacher$, $h_\teacher$ are the teacher's 
    input and hidden dimensions, respectively.
    The objective is to learn this $(x, y)$
    mapping using an identical student model
    which has a hidden dimension
    $\dim_\student = 2\times \dim_\teacher$.

    \item \textbf{Feature learning with sparse parity.}
    Sparse parity is well-studied and widely adopted for understanding neural network optimization~\citep{barak2022hidden,bhattamishra2022simplicity,edelman2023pareto,morwani2023feature,abbe2024generalization}.
    It can be viewed as learning a sparse ``feature'' embedded in a much higher ambient dimension.
    Specifically, $(\dim, k)$-parity is a function from $\vx \in \{\pm 1\}^\dim$ to $y = \prod_{i \in \gS} x_i \in \{\pm 1\}$, 
    where $\gS=\{s_1,s_2,\dots ,s_k\}\subseteq [\dim]$ is the unknown support of relevant coordinates.
    In our experiments, we set $\dim\!=\!20$, $k\!=\!6$.

    \item \textbf{Feature learning with staircase.}
    We consider a multi-feature generalization of sparse parity called the staircase function~\citep{abbe2022mergedstaircase,abbe2023leap}.
    Given input $x \in \{\pm 1\}^{\dim}$,
    the label $y$ is the sum of several parity functions, whose supports are specified by $k$ segments.
    Specifically, $y = \sum_{(s_i, e_i) \in \gP} \prod_{j = s_i}^{e_i - 1} x_{j}$,
    where $\gP = \{(s_i, e_i)\}_{i \in [k]}$ are the start (inclusive)
    and stop (exclusive) indices of a segment.
    For our experiments, we set each segment
    to be of the same size and choose $\dim=21, k=3$, i.e., $\mathcal{P} = \{(0, 7), (7, 14), (14, 21)\}$, and $y \in \{-3, -1, 1, 3\}$.

    \item \textbf{CIFAR-10}~\citep{CIFAR10}.
    The input images are flattened to a length-3072 vector and the labels are treated as 10-dimensional one-hot vectors.
    We use 400 steps in all experiments, which is sufficient for large-batch eigenbasis experiments to reach around 47\% accuracy, a reasonable performance for 2-layer MLPs.
\end{itemize}

\textit{What about using power $\power = -1$ for Adam?}
In \S\ref{sec:setup}, we introduced the power $\power \in \{-\frac{1}{2}, -1\}$ as a hyperparameter for Gauss-Newton (GN) but kept the power Adam to be $-\frac{1}{2}$, following the standard definition of the Adam algorithm.
For completeness, we experiment on Adam with $\power=-1$ on sparse parity. 
Our results in \Cref{fig:adam_power_2x2_plots} is consistent with \cite{lin24remove}, which finds that $\power=-1$ shows comparable empirical performance to the standard choice of $\power=-0.5$, especially under low-precision.

\begin{figure}[t]
  \centering
  \includegraphics[width=\textwidth]{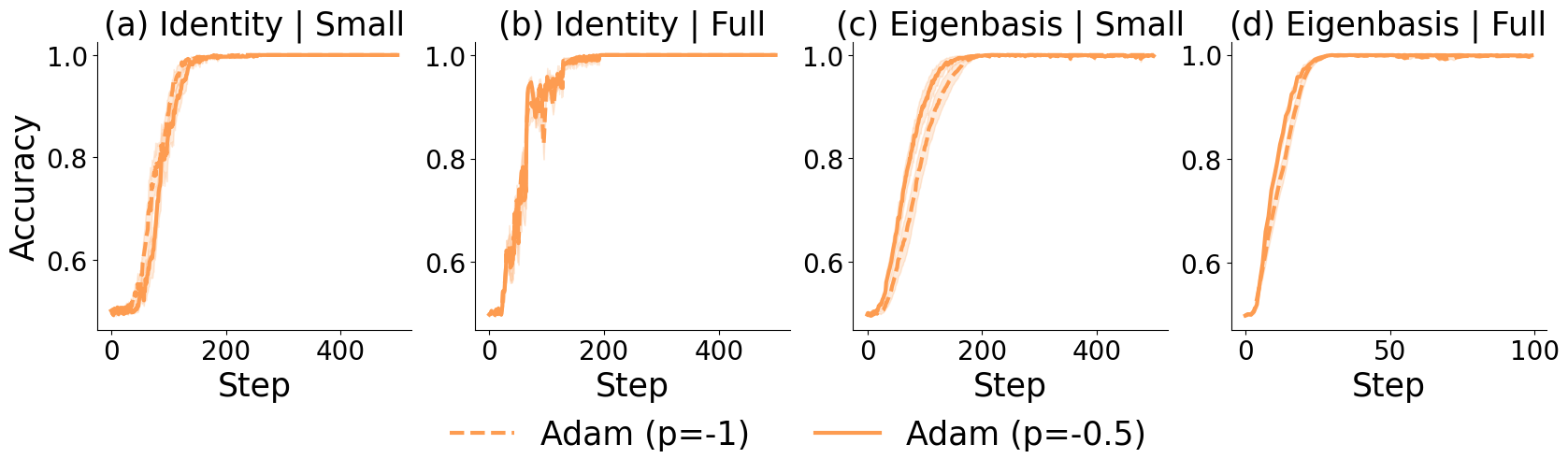}
  \caption{\textbf{Sparse parity}, comparing Adam, with power $-1$ or $-\frac{1}{2}$ for the full $2 \times 2$ grid (\Cref{tab:basis-batch-grid}).
  }
  \label{fig:adam_power_2x2_plots}
\end{figure}

\subsection{Interpolating between basis}
\label{app:interpolation}

\begin{algorithm}[ht]
  \centering
  \caption{Geodesic interpolation between bases}\label{alg:interpolation}
  \begin{tabular}{@{}r p{0.95\textwidth}@{}}
    \alglinenumber\ & \textbf{Input:} %
    full GN basis $U$, interpolation factor $\alpha$.
    \\ 
    \alglinenumber\ & \quad Compute the matrix log $K := \text{logm}(U)$. \\
    \alglinenumber\ & \quad Compute the matrix exponent $\hat U := \text{exp}(\alpha \cdot K)$. \\
    \alglinenumber\ & \quad Obtain the real part $U_{\alpha} := \text{real}(\hat U)$. \\
    \alglinenumber\ & \textbf{Output:} $U_{\alpha}$. \\
    \bottomrule
  \end{tabular}
\end{algorithm}

In \S\ref{sec:non_convex_mse}, we discussed the behavior of
GN and Adam on two kinds of basis:
identity and full-GN, depicting an incorrect
and a correct basis to precondition the gradient, respectively.
To provide a more complete picture of the effect of basis,
we provide results with more granularity with respect to the choice of basis.

Particularly, we compare GN and Adam on bases of ``intermediate'' quality 
by interpolating between the identity and full-GN basis.
Given the identity basis $I$ and the eigenbasis $\ebasis$,
we construct an interpolation $\ebasis_{\interp}$, parameterized by some interpolation factor $\interp \in \{0, 0.25, 0.5, 0.75, 1\}$, using geodesic interpolation (\Cref{alg:interpolation}).
In particular, $\ebasis_{0} = I$ and $\ebasis_{1} = \ebasis$.
Results are shown in \Cref{fig:geodesic-interp}.
In particular, Adam and $\GN^{-\frac{1}{2}}$ behave similarly under the stochastic regime across basis choices, as predicted by the theory (\S\ref{sec:stochastic}).

\begin{figure}[t]
  \centering
  \begin{subfigure}[b]{\linewidth}
    \centering
    \includegraphics[width=\linewidth]{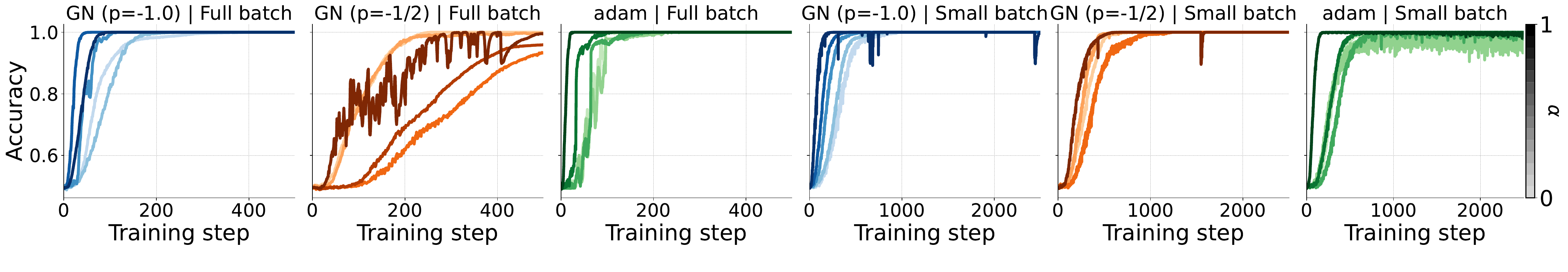}
    \caption{(20, 6)-Sparse parity}
    \label{fig:geodesic-interp-parity}
  \end{subfigure}
  \hfill
  \begin{subfigure}[b]{\linewidth}
    \centering
    \includegraphics[width=\linewidth]{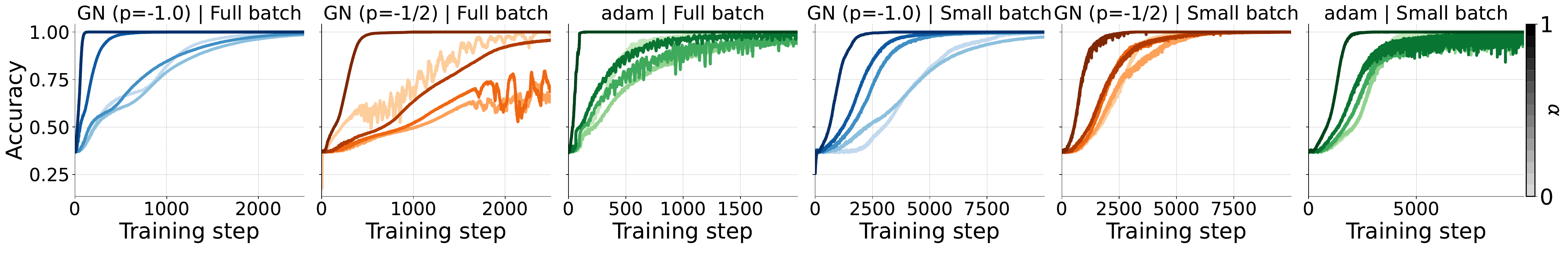}
    \caption{(21,3)-Staircase}
    \label{fig:geodesic-interp-staircase}
  \end{subfigure}
  \caption{\textbf{Basis interpolation}: Comparing GN$^{-1}$, GN$^{-1/2}$, and Adam under various bases, for parity and staircase (\S\ref{sec:non_convex_mse}).
  Each basis is obtained by a geometric interpolation between the eigenbasis (darker colors) and the identity basis (lighter colors), parameterized by a factor $\interp \in \{0, 0.25, 0.5, 0.75, 1\}$.
  }
  \label{fig:geodesic-interp}
\end{figure}

\subsection{Details for logistic experiments}
\label{app:expr_details_logistic}

\paragraph{Simulation for \S\ref{sec:logistic}}
We run simulation following the 2-layer linear network example in \S\ref{sec:logistic}.
The inputs are 2048-dimensional one-hot vectors following a power law decay, with $\nu_i := \Pr(x = e_i) \propto i^{-c}$.
We set $c = 0.6$ in the experiments.
The label distributions are set to $\py_i = p(y=1|x = e_i) = 0.75$ for all $i$.

\paragraph{Transformer experiments} This section provides details for the Transformer experiments in \S\ref{sec:expr_attn}.
The attention module shares a similar structure as the logistic regression results in \S\ref{sec:logistic}: the inner product of query and key matrices resembles the reparameterization, and the softmax function resembles the logistic function.

We consider a selection task motivated by the example in \S\ref{sec:logistic}:
The input is a sequence of $\seqlen$ Gaussian vectors followed by a length-$\dim$ one-hot vector ($\dim \geq \seqlen$) specifying which input is used in the regression task, i.e. $[x_1, \cdots, x_\seqlen, s]$, with the label given by $y = \langle\param_*, x_i \rangle$ if $s = \ve_i$.
In the experiments, we set $\seqlen = 32$ and use a 1-layer 1-head Transformer with dimension 128.
\Cref{fig:transformer} shows results the comparison of GN and Adam under the Kronecker-approximated eigenbasis with large batches (batch size = 16384), where GN shows slower convergence, consistent with the theory.
Results are aggregated over 10 seeds, and each run takes around 90min to complete.

\end{document}